\documentclass{article}

\PassOptionsToPackage{numbers, compress}{natbib}

\usepackage[preprint]{neurips_2026}
\author{
\textbf{Yuetai Li}\textsuperscript{$\clubsuit$} \;\;\;
\textbf{Fengqing Jiang}\textsuperscript{$\clubsuit$} \;\;\;  
\textbf{Yichen Feng}\textsuperscript{$\clubsuit$} \;\;\;  
\textbf{Kaiyuan Zheng}\textsuperscript{$\clubsuit$} \;\;\; 
\textbf{Basel Alomair}\textsuperscript{$\spadesuit$}\textsuperscript{$\clubsuit$}\textsuperscript{$\heartsuit$}\;\;\; \\
\textbf{Luyao Niu}\textsuperscript{$\clubsuit$} \;\;\; 
\textbf{Bhaskar Ramasubramanian }\textsuperscript{$\diamondsuit$}\;\;\;
\textbf{Linda Bushnell}\textsuperscript{$\clubsuit$} \;\;\; 
\textbf{Radha Poovendran}\textsuperscript{$\clubsuit$}\\
  \textsuperscript{$\clubsuit$}University of Washington \; 
\textsuperscript{$\diamondsuit$}Western Washington University \;\\
  \textsuperscript{$\spadesuit$}King Abdulaziz City for Science and Technology \\
  \textsuperscript{$\heartsuit$}HUMAIN \\
    \texttt{rp3@uw.edu}
  }
\usepackage[utf8]{inputenc}
\usepackage[T1]{fontenc}
\usepackage{hyperref}
\usepackage{url}
\usepackage{booktabs}
\usepackage{makecell}
\usepackage{amsfonts}
\usepackage{nicefrac}
\usepackage{microtype}
\usepackage{xcolor}

\usepackage{amsmath}
\usepackage{amssymb}
\usepackage{mathtools}
\usepackage{amsthm}
\usepackage{bbm}

\usepackage{algorithm}
\usepackage{algorithmic}

\usepackage{graphicx}
\usepackage{subcaption}
\usepackage{tikz}
\graphicspath{{./}}

\usepackage[most]{tcolorbox}

\usepackage{placeins}

\usepackage[capitalize,noabbrev]{cleveref}

\theoremstyle{plain}
\newtheorem{theorem}{Theorem}[section]
\newtheorem{proposition}[theorem]{Proposition}
\newtheorem{lemma}[theorem]{Lemma}
\newtheorem{corollary}[theorem]{Corollary}
\theoremstyle{definition}
\newtheorem{definition}[theorem]{Definition}
\newtheorem{assumption}[theorem]{Assumption}
\newtheorem{example}[theorem]{Example}
\theoremstyle{remark}
\newtheorem{remark}[theorem]{Remark}

\newcommand{\R}{\mathbb{R}}

\newcommand{\E}{\mathbb{E}}

\newcommand{\Y}{\mathcal{Y}}

\DeclareMathOperator*{\argmin}{arg\,min}
\DeclareMathOperator*{\argmax}{arg\,max}

\newcommand{\fL}{f_L}

\newcommand{\Vstar}{V^{\star}}
\newcommand{\Vmix}{V^{\mathrm{mix}}}

\newcommand{\Regret}{\mathrm{Reg}}
\newcommand{\DualGap}{\mathrm{DG}}

\newcommand{\norm}[1]{\left\lVert #1 \right\rVert}
\newcommand{\inner}[2]{\left\langle #1, #2 \right\rangle}

\newcommand{\Breg}{D_{\psi}}

\title{Polyhedral Instability Governs Regret in Online Learning}

\begin{document}

\maketitle

\begin{abstract}
Many online decision problems over combinatorial actions are addressed via convex relaxations, leading to online convex optimization with piecewise linear objectives and induced polyhedral structure. 
We show that regret in such problems is governed by \emph{polyhedral instability}: the number of changes of the active region. 
Under full information feedback and fixed partition assumptions, if $\mathrm{RS}_T$ denotes the number of region switches and $V_{\max}$ the maximum number of vertices per region, we prove $\Regret_T= \Theta(\sqrt{(1+\mathrm{RS}_T)\,T\,\log V_{\max}})$ interpolating between experts-like and dimension-dependent OCO rates. 
For online submodular--concave games under Lov\'{a}sz convexification, this reduces to the permutation-switch count $\mathrm{SC}_T$, yielding the matching rate $\Regret_T= \Theta(\sqrt{(1+\mathrm{SC}_T)\,T\,\log n})$. 
Experiments on synthetic and real combinatorial problems (shortest path, influence maximization) validate the predicted scaling and indicate that low-instability regimes can arise in practice without explicit enumeration of actions. 
\end{abstract}

\section{Introduction}\label{sec:intro}

Many online decision problems require a learner to repeatedly choose a combinatorial action while the environment changes over time. 
A network operator may need to reroute traffic as link conditions change, a public-health planner may need to reposition sensors as risk patterns shift, and a platform may need to choose seed sets as user behavior evolves. 
In all of these settings, {in each round} the learner must act before seeing the next round's outcome. The learner's performance is measured 
{by \emph{regret}: the cumulative gap between the learner's realized loss and that of an appropriate benchmark policy, such as the best fixed action or the round-wise optimum selected in hindsight}. 

These problems are often computationally challenging because the number of {feasible actions can be exponential in the problem size}. 
A widely used solution is to replace the discrete decision set by a continuous convex relaxation, enabling optimization in a space whose dimension {grows} only polynomial with problem size. 
In many such relaxations, the resulting objective is piecewise linear: the relaxation domain is partitioned into a fixed collection of polyhedral regions, and the objective is linear within each region. 
However, it remains unclear which structural property of the induced polyhedral partition can explain when regret 
{exhibits experts-like behavior \citep{FreundSchapire1997,LittlestoneWarmuth1994}, and when it instead follows the dimension-dependent rates of generic online convex optimization (OCO) \citep{Hazan2016,Zinkevich2003}}. 


In this paper, we study \emph{polyhedral instability}: the number of rounds on which the active polyhedral region changes. 
{When the learner remains in a single region, the problem reduces to an experts-like task over a small local set of vertices rather than a generic high-dimensional continuous problem. 
When the learner frequently crosses region boundaries, the local linear model changes repeatedly, and the advantage of the polyhedral structure diminishes.} 
Under fixed-partition and full-information assumptions, {we show regret bounds that interpolate between these two regimes: experts-like rates under low instability, and ambient-dimension-dependent rates under high instability.} 

The key message is that {the difficulty of polyhedral online learning is not determined by ambient dimension alone.} 
Two {instances} with the same dimension can have very different {regret complexity depending on how often the active region changes.} 
This perspective also helps explain empirical behavior that standard theory {does not capture}: some large combinatorial problems behave surprisingly like small experts problems, while others do not. 
Because polyhedral partitions arise in many combinatorial relaxations, this viewpoint applies broadly, including to the Lov\'{a}sz extension for submodular optimization~\citep{Bach2013,Lovasz1983, HazanKale2012} and other relaxations built from polyhedra~\citep{Schrijver2003}.

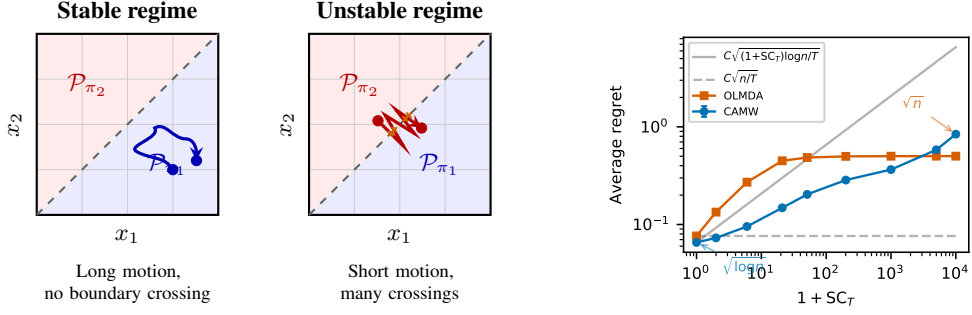
\begin{figure}[t]
\centering
\begin{minipage}[b]{0.62\textwidth}
\centering
\begin{tikzpicture}[scale=2.4, >=stealth]
  \begin{scope}[xshift=0cm]
    \fill[blue!8] (0,0) -- (1,0) -- (1,1) -- cycle;
    \fill[red!8] (0,0) -- (0,1) -- (1,1) -- cycle;
    \draw[gray!40, very thin] (0,0) grid[step=0.25] (1,1);
    \draw[black, thick] (0,0) rectangle (1,1);
    \draw[black!60, thick, dashed] (0,0) -- (1,1);
    \node[blue!70!black, font=\footnotesize] at (0.72,0.28) {$\mathcal{P}_{\pi_1}$};
    \node[red!70!black, font=\footnotesize] at (0.28,0.72) {$\mathcal{P}_{\pi_2}$};
    \draw[blue!70!black, very thick, ->]
      (0.75,0.25) to[out=120,in=-30] (0.55,0.35) to[out=150,in=-60] (0.65,0.48)
      to[out=30,in=180] (0.82,0.42) to[out=0,in=90] (0.88,0.30);
    \filldraw[blue!70!black] (0.75,0.25) circle (0.8pt);
    \filldraw[blue!70!black] (0.88,0.30) circle (0.8pt);
    \node[font=\footnotesize] at (0.5,-0.12) {$x_1$};
    \node[font=\footnotesize, rotate=90] at (-0.12,0.5) {$x_2$};
    \node[font=\small\bfseries] at (0.5,1.12) {Stable regime};
    \node[font=\scriptsize, text width=2.2cm, align=center] at (0.5,-0.38) {Long motion,\\ no boundary crossing};
  \end{scope}
  \begin{scope}[xshift=1.5cm]
    \fill[blue!8] (0,0) -- (1,0) -- (1,1) -- cycle;
    \fill[red!8] (0,0) -- (0,1) -- (1,1) -- cycle;
    \draw[gray!40, very thin] (0,0) grid[step=0.25] (1,1);
    \draw[black, thick] (0,0) rectangle (1,1);
    \draw[black!60, thick, dashed] (0,0) -- (1,1);
    \node[blue!70!black, font=\footnotesize] at (0.72,0.28) {$\mathcal{P}_{\pi_1}$};
    \node[red!70!black, font=\footnotesize] at (0.28,0.72) {$\mathcal{P}_{\pi_2}$};
    \draw[red!70!black, very thick, ->]
      (0.38,0.52) -- (0.52,0.38) -- (0.42,0.55) -- (0.58,0.45)
      -- (0.48,0.58) -- (0.62,0.48);
    \filldraw[red!70!black] (0.38,0.52) circle (0.8pt);
    \filldraw[red!70!black] (0.62,0.48) circle (0.8pt);
    \foreach \x/\y in {0.45/0.45, 0.47/0.47, 0.53/0.53, 0.55/0.55} {
      \draw[orange!80!black, thick] (\x-0.015,\y-0.015) -- (\x+0.015,\y+0.015);
      \draw[orange!80!black, thick] (\x-0.015,\y+0.015) -- (\x+0.015,\y-0.015);
    }
    \node[font=\footnotesize] at (0.5,-0.12) {$x_1$};
    \node[font=\footnotesize, rotate=90] at (-0.12,0.5) {$x_2$};
    \node[font=\small\bfseries] at (0.5,1.12) {Unstable regime};
    \node[font=\scriptsize, text width=2.2cm, align=center] at (0.5,-0.38) {Short motion,\\ many crossings};
  \end{scope}
\end{tikzpicture}
\end{minipage}%
\hfill
\begin{minipage}[b]{0.35\textwidth}
\centering
\includegraphics[width=\textwidth]{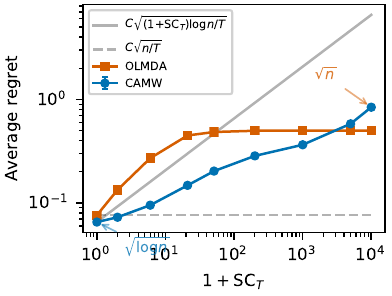}
\end{minipage}
\caption{The key distinction is not how far the learner moves, but whether it crosses region boundaries.
\textbf{Left}: motion inside one region keeps the same local linear model.
\textbf{Center}: repeated boundary crossings keep changing the effective local problem.
\textbf{Right}: the experiments reflect this transition between stable and unstable regimes (\Cref{sec:experiments}).}
\label{fig:cells}
\end{figure}

As a concrete and practically important instance, we specialize polyhedral instability to online submodular--concave games under Lov\'{a}sz convexification. 
In this setting, instability corresponds to the number of times the active Lov\'asz cell changes, 
equivalently the number of equilibrium permutation switches. 
This viewpoint also leads to a practical algorithm: use multiplicative-weights-style learning inside one cell, and only restart or transfer information when the active cell changes. 
The resulting algorithms are simple, interpretable, and directly tied to the geometry of the problem.


To summarize, this paper makes the following contributions.
\begin{itemize}
\item We identify \emph{polyhedral instability} as a structural feature that yields instability-parameterized regret bounds for online learning with piecewise-linear objectives under fixed-partition assumptions.
We show how this applies both to fixed polyhedral partitions and to the specific case of Lov\'{a}sz convexification for submodular--concave games.
\item We design region-aware multiplicative-weights algorithms that exploit stable geometric regimes rather than treating the objective as a generic black-box convex loss.
\item We validate the predicted scaling on controlled synthetic games and provide preliminary evidence on shortest-path and influence-maximization tasks that low-instability regimes can arise beyond designed instances.
\end{itemize}

\section{Problem Setup and Preliminary Background}\label{sec:prelim}

\paragraph{Online Learning Setup.}
We study online learning over $T$ rounds for combinatorial problems whose convex relaxation induces a fixed polyhedral partition. 
At each round $t$, the learner chooses a decision in the relaxation, the environment reveals the round-$t$ objective, and performance is measured by regret against the appropriate round-wise optimal comparator. 

\paragraph{Submodular--Concave Games.}
One can instantiate an online learning problem with the submodular--concave game model.
Let $[n] = \{1,\dots,n\}$.
A submodular-concave game consists of a ground set~$[n]$, a compact convex set $\Y \subseteq \R^m$, and a payoff function $f \colon 2^{[n]} \times \Y \to \R$ such that $f(\cdot, y)$ is submodular for every $y \in \Y$ and $f(S, \cdot)$ is concave for every subset $S \subseteq [n]$.

\paragraph{Lov\'{a}sz Extension.}
For $x \in [0,1]^n$ sorted as $x_{\pi(1)} \ge \cdots \ge x_{\pi(n)}$, the Lov\'{a}sz extension is
\begin{equation}\label{eq:lovasz}
  \fL(x, y) = \sum_{i=0}^{n} (x_{\pi(i)} - x_{\pi(i+1)}) \, f(C_i, y),
\end{equation}
where $C_i = \{\pi(1), \dots, \pi(i)\}$.
If $f(\cdot, y)$ is submodular, then $\fL(\cdot, y)$ is convex and $\min_x \fL(x,y) = \min_S f(S,y)$~\citep{Lovasz1983}.
The mixed minimax value equals $\Vstar := \min_{x \in [0,1]^n} \max_{y \in \Y} \fL(x,y)$; optimal mixed strategies are supported on chains of at most $n{+}1$ nested subsets (\Cref{lem:chain-support} in \Cref{app:background}).
For each permutation $\pi$ of $[n]$, the corresponding \emph{Lov\'asz cell} is the region
\[
  \mathcal{P}_\pi = \{x \in [0,1]^n : x_{\pi(1)} \ge \cdots \ge x_{\pi(n)}\},
\]
so the cube is partitioned into Lov\'asz cells indexed by coordinate orderings, and $\fL(\cdot,y)$ is linear on each such cell.

\paragraph{Online protocol.}
At each round $t = 1, \dots, T$: the learner selects $(x_t, y_t) \in [0,1]^n \times \Y$; nature reveals~$f_t$; the learner incurs the \emph{saddle-point regret}: 
\begin{equation}\label{eq:sp-regret}
  \Regret_T = \sum_{t=1}^{T} \bigl[ \fL^t(x_t, y_t^\star) - \fL^t(x_t^\star, y_t) \bigr],
\end{equation}
where $(x_t^\star,y_t^\star)$ is a saddle point of the round-$t$ game, i.e., $\fL^t(x_t^\star,y) \le \fL^t(x_t^\star,y_t^\star)
  \le \fL^t(x,y_t^\star)
  \quad \forall x,y$.
We call~\eqref{eq:sp-regret} \emph{dynamic saddle-point regret}, as the comparator $(x_t^\star, y_t^\star)$ varies with~$t$.
It measures how much the learner loses relative to the roundwise equilibrium sequence, rather than to one fixed hindsight decision.
The stability quantities below are therefore structural assumptions on the game sequence induced by these equilibria, not choices made by the learner after the fact.
When the equilibrium sequence changes cells frequently, the bound recovers dimension-dependent behavior; when it remains in few cells, the bound captures the easier dynamic comparator path.

\begin{assumption}\label{assump:bounded}
Let $g_t^x := \nabla_x \fL^t(x_t, y_t)$ and $g_t^y := \nabla_y \fL^t(x_t, y_t)$ denote the gradients at round~$t$.
There exist $L_x, L_y, M > 0$ with $\norm{g_t^x}_2 \le L_x$, $\norm{g_t^y}_2 \le L_y$, $|f^t(S,y)| \le M$ for all $t, S, y$.
\end{assumption}

Additional structural background and assumption scaling are presented in Appendix \ref{app:background-foundations}.

\section{Permutation Stability Governs Regret}\label{sec:curvature}

This section instantiates polyhedral instability for submodular--concave games under Lov\'asz convexification.
We analyze how the equilibrium cell sequence enters regret bounds for these games.
Our analysis follows two stages.
We start with an oracle version which knows the equilibrium permutation.
Next, we develop a proxy algorithm to remove the oracle by using an observable best-response permutation.
\Cref{sec:lower} lifts this principle to arbitrary polyhedral partitions.


\subsection{Permutation Switches Drive Regret}\label{sec:sc-def}

The Lov\'asz extension is linear on each Lov\'asz cell~\citep{Bach2013, Lovasz1983}, with gradient determined by the nested chain $C_i = \{\pi(1), \dots, \pi(i)\}$.
For $x \in [0,1]^n$, let $\pi$ denote the permutation that sorts $x$ in descending order, so that $x \in \mathcal{P}_{\pi}$. 
If the equilibrium remains in {a single} cell $\mathcal{P}_\pi$ throughout the horizon, then the minimizer of $\min_x \max_y \fL^t(x,y)$ is effectively solving an experts problem over only the $n{+}1$ chain vertices $\{\mathbf{1}_{C_i}\}_{i=0}^{n}$ of that cell, and multiplicative weights over those vertices attains the $O(\sqrt{\log n / T})$ rate~\citep{FreundSchapire1997, LittlestoneWarmuth1994}. 
{A challenge arises when equilibrium crosses cell boundaries, since each crossing changes the local linear model, the associated chain, and hence the induced experts problem. We quantify this complexity through the number of equilibrium cell changes.}

\begin{definition}[Permutation-switch count]\label{def:switch-count}
Let $x_t^\star = \argmin_{x} \max_{y} \fL^t(x,y)$ be a minimax equilibrium at round~$t$, and let $\pi_{x_t^\star}$ be the permutation that sorts $x_t^\star$.
The \emph{permutation-switch count} is
\begin{equation}\label{eq:switch-count}
  \mathrm{SC}_T := \bigl|\bigl\{t \in [T{-}1] : \pi_{x_{t+1}^\star} \ne \pi_{x_t^\star}\bigr\}\bigr|.
\end{equation}
\end{definition}
{$\mathrm{SC}_T$ is an intrinsic property of the game sequence and does not depend on the learner's actions. In influence maximization, it corresponds to changes in the ranking of nodes by marginal influence; in sensor placement, it measures changes in the ordering of marginal coverage gains.}


\paragraph{Relation to Existing Adaptive Measures.}
Existing adaptive measures emphasize different phenomena.
Path length and gradual-variation measures track Euclidean motion or temporal drift in the losses~\citep{ChiangEtAl2012,ZhangLuZhou2018}, while shifting-expert bounds count changes of the best comparator in a \emph{fixed} expert class~\citep{HerbsterWarmuth1998}. 
{In contrast, here the relevant expert class itself changes when the equilibrium crosses a Lov\'asz cell boundary. Thus $\mathrm{SC}_T$ captures a distinct form of combinatorial instability and is incomparable with path length, gradient variation, and classical switching counts.} 
Formal separations and related structural properties are in Appendix \ref{app:separations}.


\subsection{Oracle Reference: Cell-Aware Multiplicative Weights}\label{sec:oracle-camw}

We next show $\mathrm{SC}_T$ controls regret.
We first analyze an \emph{oracle} version of Cell-Aware Multiplicative Weights (CAMW) that is revealed the equilibrium permutation $\pi_{x_t^\star}$ at each round.
This {result} should be {interpreted} as a diagnostic benchmark: $\mathrm{SC}_T$ is an intrinsic but generally inaccessible property of the comparator sequence. 
{Whenever the equilibrium permutation remains fixed, the minimizer faces a static experts problem over the associated $n+1$ chain vertices, for which multiplicative weights is sufficient. When the permutation changes, CAMW restarts on the new chain. Consequently, the horizon decomposes into $\mathrm{SC}_T+1$ epochs, each corresponding to a single Lov\'asz cell.} 

\begin{theorem}[Cell-stability-adaptive regret]\label{thm:curvature-regret}
Let $D_\Y$ denote the Bregman diameter of $\Y$. 
Under Assumption \ref{assump:bounded}, CAMW  with mirror ascent for the maximizer of $\min_{x} \max_{y} \fL^t(x,y)$ achieves:
\begin{equation}\label{eq:cell-bound}
  \frac{1}{T} \Regret_T \le O\!\left(\sqrt{\frac{(1 + \mathrm{SC}_T) \cdot \log n}{T}}\right) + \frac{L_y D_\Y}{\sqrt{T}}.
\end{equation}
\end{theorem}
\Cref{thm:curvature-regret} shows that regret scales with the number of Lov\'asz cells visited, not with the full dimension of $[0,1]^n$.
The proof and supporting lemmas are given in Appendix \ref{app:upper-bounds}.

\begin{corollary}[Extreme regimes]\label{cor:extremes}
When $\mathrm{SC}_T = 0$, \Cref{thm:curvature-regret} recovers $O(\sqrt{\log n / T})$. 
When $\mathrm{SC}_T = O(n/\log n)$ (maximally unstable), it recovers the rate of Online Lov\'asz Mirror Descent-Ascent (OLMDA) $O(\sqrt{n/T})$.
\end{corollary}

\subsection{Boundary-Aware Implementation}\label{sec:implementable}

The oracle reference is typically unavailable to the learner and is not used by the implementable algorithm.
To obtain an online method, we replace the inaccessible equilibrium permutation with an observable best-response permutation.
This keeps the same high-level logic as the oracle proof while grounding the switches in quantities the learner can compute online.

Cold-start CAMW, which restarts MW whenever the observed cell changes, is the core method behind the worst-case rate. 
{We also consider a warm-start refinement motivated by the substantial overlap between adjacent Lov\'asz cells' chain structures. Under an adjacent transposition, only a local portion of the chain changes, while most chain elements remain unchanged.}
We develop geometric CAMW (\Cref{alg:camw}) that exploits this overlap through \emph{chain-overlap transfer}: when the observed permutation changes, the algorithm {retains} weight of shared chain elements and assigns fresh mass only to genuinely new ones.
The transfer parameter $\alpha$ acts as a probability floor: every new-chain expert receives at least $\alpha/(n{+}1)$ mass, which preserves the MW initialization guarantee, while inherited mass improves constants when consecutive chains overlap.
Implementation details, the OLMDA baseline, and the chain-overlap transfer guarantee are given in Appendix \ref{app:algorithms}.

\begin{algorithm}[t]
\caption{Geometric CAMW}
\label{alg:camw}
\begin{algorithmic}[1]
\REQUIRE Step sizes $\eta_x, \eta_y > 0$; transfer parameter $\alpha \in (0,1)$; initial $y_1 \in \Y$
\STATE Initialize MW: $p_1 = (1/(n{+}1), \dots, 1/(n{+}1))$; $\pi_0 \gets \mathrm{id}$
\FOR{$t = 1, \dots, T$}
  \STATE $\pi_t \gets \pi^{\mathrm{BR}}_t(y_t)$ \hfill 
  \IF{$\pi_t \ne \pi_{t-1}$}
    \STATE $\mathcal{S} \gets \{j : C_j^{\pi_t} = C_j^{\pi_{t-1}}\}$ \hfill 
    \STATE \textbf{Transfer}: $p_{t,j} \gets (1{-}\alpha)\,p_{t-1,j} + \tfrac{\alpha}{n+1}$ for $j \in \mathcal{S}$;\; $p_{t,j} \gets \tfrac{\alpha}{n+1}$ for $j \notin \mathcal{S}$
    \STATE Normalize: $p_t \gets p_t / \|p_t\|_1$
  \ENDIF
  \STATE $x_{t, \pi_t(j)} = \sum_{i \ge j} p_{t,i}$; \textbf{play} $(x_t, y_t)$; \textbf{observe} $f_t$
  \STATE MW update: $p_{t+1,i} \propto p_{t,i} \exp(-\eta_x \, f_t(C_i^{\pi_t}, y_t))$
  \STATE Mirror ascent: $y_{t+1} = \argmax_{y \in \Y} \{\inner{\nabla_y \fL^t(x_t, y_t)}{y} - \frac{1}{\eta_y} D_\varphi(y, y_t)\}$
\ENDFOR
\end{algorithmic}
\end{algorithm}

\begin{theorem}[Geometric WS-CAMW]\label{thm:implementable-camw}
Under Assumption \ref{assump:bounded}, \Cref{alg:camw} achieves:
\begin{equation}\label{eq:impl-bound}
  \frac{1}{T} \Regret_T \le O\!\left(\sqrt{\frac{(1 + \widehat{\mathrm{SC}}_T) \cdot \log n}{T}}\right) + \frac{L_y D_\Y}{\sqrt{T}},
\end{equation}
where $\widehat{\mathrm{SC}}_T = |\{t : \pi_{t+1}^{\mathrm{BR}} \ne \pi_t^{\mathrm{BR}}\}|$ is the observed best-response switch count.
\end{theorem}
This is an online, parameter-free bound with respect to the observed count $\widehat{\mathrm{SC}}_T$.
Recovering the oracle rate in terms of $\mathrm{SC}_T$ additionally requires the tracking/gap condition stated in Appendix \ref{app:tracking}; without such separation, $\widehat{\mathrm{SC}}_T$ can be larger than $\mathrm{SC}_T$.

\section{Lower Bounds and the Polyhedral Principle}\label{sec:lower}

This section gives matching lower-bound evidence for the $\mathrm{SC}_T$ dependence and then states the corresponding fixed-partition polyhedral result.
The broader result is limited to the full-information, fixed finite partition model made explicit below.

\subsection{Cell Transitions Force Unavoidable Regret}

We next show that the $\mathrm{SC}_T$ term cannot be removed in the worst case.
Any prescribed switch count can be encoded by an adversary into a submodular--concave game that forces matching regret.
The construction partitions the horizon into $\mathrm{SC}_T+1$ epochs separated by permutation shifts.
At a high level, the adversary splits time into $s{+}1$ epochs, selects a different Lov\'asz cell for each epoch, and places independent Rademacher expert losses on the chain vertices inside that cell.
Each epoch therefore contains a fresh experts lower-bound instance, and summing the epoch lower bounds gives $\Omega(\sum_j \sqrt{T_j\log n})=\Omega(\sqrt{(s{+}1)T\log n})$ by Cauchy--Schwarz.

\begin{theorem}[$\mathrm{SC}_T$-dependent lower bound]\label{thm:lower-switch}
For any $s \in \{0, \dots, T{-}1\}$ and $n \ge 3$, there exists a sequence of submodular--concave games with $\mathrm{SC}_T = s$ such that
$\frac{1}{T}\E[\Regret_T] \ge \Omega\!\bigl(\sqrt{(1+s)\log n\,/\,T}\bigr)$.
\end{theorem}
The full lower-bound construction and the proof are given in Appendix \ref{app:lower-bounds}.


\begin{corollary}[Matching switch-dependent rate]\label{cor:canonical}
Combining \Cref{thm:curvature-regret,thm:lower-switch}: $\Regret_T = \Theta\!\bigl(\sqrt{(1+\mathrm{SC}_T)\,T\,\log n}\bigr)$.
\end{corollary}
\Cref{cor:canonical} shows that stability is both sufficient and necessary for improved rates.
The worst-case $\Omega(\sqrt{n/T})$ lower bound is recovered when $\mathrm{SC}_T = \Theta(n/\log n)$, via reduction to online linear optimization with Rademacher losses~\citep{AbernethyEtAl2008}.


\subsection{General Polyhedral Principle}\label{sec:general-polyhedral}

We now state the fixed-partition polyhedral analogue.
This setting covers several applications with finite polyhedral cells, but it does not cover time-varying partitions, smooth relaxations, bandit feedback, or cases where the active region cannot be identified under the stated feedback model.
Examples include \emph{online shortest-path} on a DAG with $N$ paths over $d$ edges (each path defines a vertex of the path polytope), \emph{online min-cost flow} with $d$ arcs and $V_{\max}$ basic feasible solutions, and \emph{structured prediction} over label polytopes.
In each case, the decision space $\mathcal{X} \subset \R^d$ partitions into polyhedral regions $\{\mathcal{R}_k\}_{k=1}^K$ with losses linear on each region.
Define the \emph{region-switch count} $\mathrm{RS}_T := |\{t : \mathcal{R}(x_t^\star) \ne \mathcal{R}(x_{t+1}^\star)\}|$ and $V_{\max} := \max_k |V(\mathcal{R}_k)|$.
We bound the regret for polyhedral OCO as follows.

\begin{theorem}[Regret for fixed polyhedral partitions]\label{thm:pl-oco}
Consider full-information piecewise-linear losses over a fixed finite polyhedral partition.
Suppose Assumption \ref{assump:bounded} holds, then
$\Regret_T = \Theta\!\bigl(\sqrt{(1+\mathrm{RS}_T)\,T\,\log V_{\max}}\bigr)$.
\end{theorem}

\begin{remark}\label{rem:geometry-dependence}
In the ideal full-information model of \Cref{thm:pl-oco}, the regret analysis runs MW over the vertices of the active region and uses only bounded losses and the number of available local vertices.
Facet angles, adjacency graph geometry, and conditioning can affect region detection, computational cost and numerical stability, but they are not needed for the regret bound in \Cref{thm:pl-oco}.
Extending the theorem to noisy detection, time-varying partitions, or computationally constrained region identification may require such geometric quantities.
\end{remark}



\subsection{Unify the Experts and OCO Regimes}

We compare the fixed-partition MW rate with the usual dimension-dependent OCO rate.
When region switches are rare, MW over active-region vertices can be preferable; when switches are frequent, continuous methods can dominate.

\begin{theorem}\label{thm:phase-transition}
For piecewise-linear losses over a fixed polyhedral partition with $V_{\max}$ vertices per region in $\R^d$:
\begin{equation}\label{eq:phase-transition}
  \frac{1}{T}\Regret_T^\star(\mathrm{RS}_T)
  = \Theta\!\left(\min\!\left\{
  \sqrt{\frac{(1{+}\mathrm{RS}_T)\log V_{\max}}{T}},\;
  \sqrt{\frac{d}{T}}
  \right\}\right),
\end{equation}
with crossover at $\mathrm{RS}_T^\star = d / \log V_{\max}$.
\end{theorem}

This connects the experts regime ($\sqrt{T \log N}$) and OCO regime ($\sqrt{dT}$) through region-switch complexity under the assumptions of \Cref{thm:pl-oco}.
Below the threshold, MW-with-restarts over region vertices is optimal; above it, continuous first-order methods dominate.
For Lov\'{a}sz convexification ($d = n$, $V_{\max} = n{+}1$), the threshold is $\mathrm{RS}_T^\star \approx n/\log n$.

\begin{figure}[t]
\centering
\includegraphics[width=0.55\textwidth]{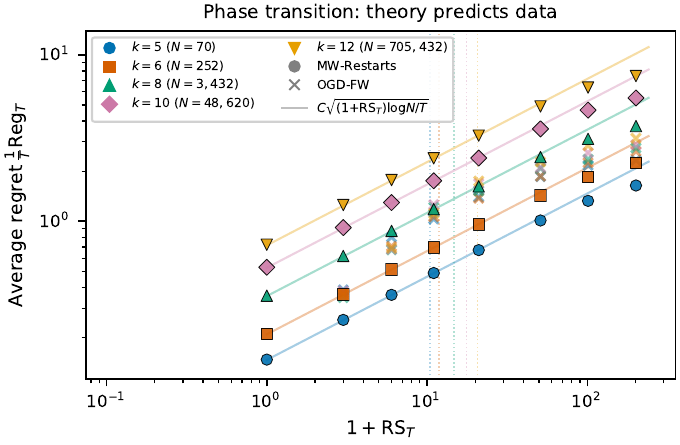}
\vspace{-0.5em}
\caption{Phase transition (\Cref{thm:phase-transition}), validated on shortest-path games.
Solid lines: theoretical rate $C\sqrt{(1{+}\mathrm{RS}_T)\log N / T}$, fitted per grid size.
Circles: MW-with-restarts regret tracks the theory across five grid sizes ($k = 5$ to $12$, up to $N = 705{,}432$ paths).
Crosses: OGD-FW regret.
Vertical dotted lines mark the predicted crossover $\mathrm{RS}_T^\star = d/\log N$ for each~$k$; at high $\mathrm{RS}_T$, OGD becomes competitive, confirming the phase transition.}
\label{fig:phase-transition}
\end{figure}

\Cref{thm:pl-oco} applies to the fixed-partition full-information piecewise-linear losses described above, not just Lov\'{a}sz convexifications.
We have two immediate corollaries: 
\textbf{(a)}~\emph{online shortest-path} on a DAG with $d$ edges and $N$ paths gives tight regret $\Theta(\sqrt{(1{+}\mathrm{RS}_T)\,T\,\log N})$; \textbf{(b)}~\emph{online min-cost flow} with $d$ arcs and $V_{\max}$ basic feasible solutions gives $\Theta(\sqrt{(1{+}\mathrm{RS}_T)\,T\,\log V_{\max}})$ under $\mathrm{RS}_T$ basis pivots.
In both cases $\mathrm{RS}_T$ captures combinatorial instability rather than Euclidean displacement (Appendix \ref{app:incomparability}).
We validate the shortest-path instance experimentally in \Cref{sec:experiments}.
Taken together, the region-switch count gives a concise interpolation parameter in this fixed-partition model.


\section{Experiments}\label{sec:experiments}

{We evaluate whether instability measures identified by the theory predict empirical regret in structured online-learning problems.}
The experiments span Lov\'{a}sz games, online shortest-path on grid DAGs, and influence maximization on SNAP networks.
{Synthetic sweeps are controlled scaling checks rather than evidence by themselves of universality; real-data experiments are intended as preliminary evidence that low-instability regimes can arise outside designed instances.}
{Additional experiments, ablations, revision-time diagnostics, and implementation details appear in \Cref{app:experiments}.}

\subsection{Experimental Setup}

\paragraph{Datasets and domains.}
We use three experimental domains.
\begin{itemize}
    \item Lov\'{a}sz games with controlled permutation-switch counts, using a stability sweep at $n=20$, $T=10{,}000$, and $\mathrm{SC}_T \in \{0,1,5,20,50,200,1{,}000,5{,}000,9{,}999\}$, together with a separate dimension-scaling study with $n \in \{10,20,50,100,200,500,1{,}000\}$ and $\mathrm{SC}_T = 0$. {Appendix \ref{app:revision-diagnostics} reports additional Lov\'{a}sz diagnostics for warm starts, baseline comparisons, oracle robustness, and switch detection.}
    \item Online shortest-path on $k \times k$ grid DAGs with up to $N = 705{,}432$ source-to-sink paths and $T = 50{,}000$ rounds.
    \item Four SNAP influence-maximization networks: Karate Club, Email-Eu, Wiki-Vote, and Epinions, with $n$ up to $2{,}744$ and $T$ up to $20{,}000$.
\end{itemize}

\paragraph{Baselines.}
We compare CAMW and geometric WS-CAMW against OLMDA, OGD for Lov\'{a}sz games, OGD-FW for shortest-path, Fixed-Share MW~\citep{HerbsterWarmuth1998}, SAOL, and ZO-EG~\citep{FarzinEtAl2025}.
{These baselines cover continuous OCO, structure-agnostic experts-style methods, and a strongly-adaptive regret comparator, while keeping the comparison aligned with each domain's native optimization geometry.}
The choice of continuous baseline depends on the geometry of the feasible set.
In the Lov\'{a}sz experiments, the minimizer acts directly on the cube $[0,1]^n$, so Euclidean projection is trivial and plain OGD is the natural ambient-dimension baseline for testing the predicted $\sqrt{n}$ dependence.
In shortest-path, by contrast, the continuous relaxation is the path/flow polytope over graph edges; exact Euclidean projection onto this polytope at every round is substantially more involved.
Step sizes follow the choices summarized in \Cref{app:hyperparams}, unless an appendix ablation states otherwise.

\paragraph{Metrics.}
We use three metrics to validate polyhedral instability:
\begin{itemize}
    \item \textbf{Per-round regret.} Our primary performance metric is the per-round saddle-point regret $\frac{1}{T}\Regret_T$, where $\Regret_T$ is defined in Eq. \eqref{eq:sp-regret}.

    \item \textbf{Theory-normalized regret.} To test per-round regret rate predicted by Theorems \ref{thm:curvature-regret} and \ref{thm:phase-transition}, we divide $\bar{R}_T$ by the appropriate theoretical rate. For structure-aware methods we report
    \[
      \bar{R}_T^{\mathrm{struct}}
      := \frac{\bar{R}_T}{\sqrt{(1+\kappa_T)\log V_{\mathrm{eff}}/T}},
    \]
    where $\kappa_T$ is the relevant instability count and $V_{\mathrm{eff}}$ is the effective vertex count ($V_{\mathrm{eff}} = n$ for Lov\'{a}sz games and SNAP, $V_{\mathrm{eff}} = N$ for shortest-path). For continuous baselines we report
    \[
      \bar{R}_T^{\mathrm{cont}}
      := \frac{\bar{R}_T}{\sqrt{d_{\mathrm{eff}}/T}},
    \]
    where $d_{\mathrm{eff}}$ is the ambient dimension of the continuous relaxation ($d_{\mathrm{eff}} = n$ in Lov\'{a}sz games, $d_{\mathrm{eff}} = d$ in shortest-path).

    \item \textbf{Instability statistic.} Let $\kappa_T \in \{\mathrm{SC}_T,\widehat{\mathrm{SC}}_T,\mathrm{SC}_T^{\mathrm{BR}},\mathrm{RS}_T\}$ be the structural quantity driving regret,
    with
    \begin{align*}
    &\mathrm{SC}_T
      = \bigl|\{t \in [T{-}1] : \pi_{x_{t+1}^\star} \ne \pi_{x_t^\star}\}\bigr|,
      \quad
      \widehat{\mathrm{SC}}_T
      = \bigl|\{t \in [T{-}1] : \pi_{t+1}^{\mathrm{BR}} \ne \pi_t^{\mathrm{BR}}\}\bigr|\\
      &\mathrm{SC}_T^{\mathrm{BR}}
      = \bigl|\{t \in [T-1] : \widetilde{\pi}_{t+1}^{\mathrm{BR}} \ne \widetilde{\pi}_t^{\mathrm{BR}}\}\bigr|
      \quad 
      \mathrm{RS}_T
      = \bigl|\{t \in [T{-}1] : \mathcal{R}(x_{t+1}^\star) \ne \mathcal{R}(x_t^\star)\}\bigr|,
    \end{align*}
    where $\widetilde{\pi}_t^{\mathrm{BR}}$ is the offline best-response permutation recomputed from logged 
    iterate $y_t$.

\end{itemize}


\subsection{Experimental Results}

\paragraph{Finding 1: regret grows with instability, not with elapsed time alone.}
From \Cref{fig:stability-sweep}, the stability sweep over $\mathrm{SC}_T$ is consistent with the $\sqrt{(1+\mathrm{SC}_T)\log n / T}$ scaling from \Cref{thm:curvature-regret}.
CAMW's absolute regret increases sublinearly with instability, while regret normalized by $\sqrt{(1+\mathrm{SC}_T)\log n / T}$ stays nearly constant in the center panel.
The left panel gives a log-log slope of $0.54$ for absolute regret against $1+\mathrm{SC}_T$, close to the theoretical exponent $0.5$, and the right panel shows that tracked switch counts match the prescribed instability level.
Together these plots identify cell changes, rather than elapsed time alone, as the quantity predicting difficulty in the Lov\'{a}sz setting.

\paragraph{Finding 2: stable polyhedral structure yields a $\sqrt{\log n}$ dependence instead of $\sqrt{n}$.}
As shown in \Cref{fig:dimension}, CAMW's absolute regret grows much more slowly than the continuous baselines in the stable regime.
The left panel shows CAMW divided by $\sqrt{\log n / T}$ staying nearly flat with coefficient of variation $0.12$, while the right panel shows OGD divided by $\sqrt{n/T}$ staying flat with coefficient of variation $0.15$.
At $n = 1{,}000$, the plotted gap between CAMW and the continuous baselines exceeds an order of magnitude.
These curves empirically separate the experts-like and continuous-regret regimes predicted by the theory.

\begin{figure}[t]
\centering
\begin{subfigure}[t]{\textwidth}
\centering
\includegraphics[width=0.8\textwidth]{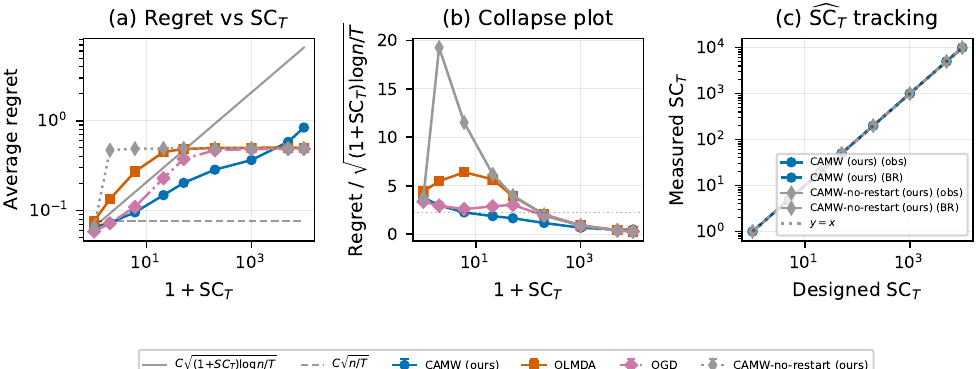}
\caption{{Stability sweep. Left: CAMW interpolates from $O(\sqrt{\log n/T})$ toward the dimension-dependent regime. Center: controlled scaling check under theory normalization. Right: online $\mathrm{SC}_T$ tracking on designed instances.}}
\label{fig:stability-sweep}
\end{subfigure}
\vspace{0.3em}
\begin{subfigure}[t]{\textwidth}
\centering
\includegraphics[width=0.8\textwidth]{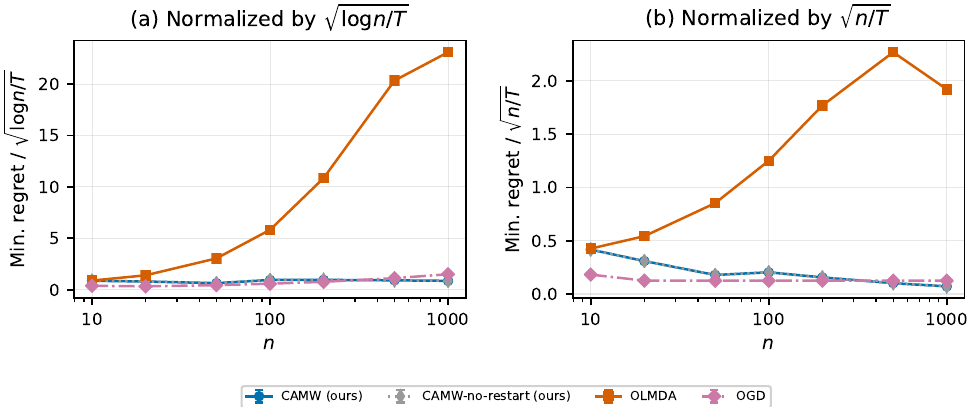}
\caption{Dimension scaling ($\mathrm{SC}_T=0$, $T=20{,}000$, 5 seeds, $n \in \{10,20,50,100,200,500,1{,}000\}$). Synthetic instances are normalized to unit loss range.
Left: CAMW's minimizer regret is flat when normalized by $\sqrt{\log n / T}$ (CV $= 0.12$), confirming $\sqrt{\log n}$ scaling.
Right: OGD is flat when normalized by $\sqrt{n/T}$ (CV $= 0.15$); OLMDA grows super-linearly, confirming $\sqrt{n}$ scaling for both baselines.}
\label{fig:dimension}
\end{subfigure}
\caption{Synthetic experiments confirm the $\sqrt{(1{+}\mathrm{SC}_T)\log n / T}$ scaling of \Cref{thm:curvature-regret}. CAMW (ours) is shown in blue.}
\label{fig:synthetic}
\end{figure}

\paragraph{{Finding 3: fixed-partition instability model is consistent with shortest-path behavior.}}
As shown in \Cref{fig:shortest-path}, online shortest-path on $k \times k$ grid DAGs follows the same instability-based pattern.
MW-with-restarts has absolute regret that grows with $\mathrm{RS}_T$, while normalized regret stays stable, with log-log slope $0.44$ for the unnormalized curve in Fig. \ref{fig:shortest-path}.
Fig. \ref{fig:shortest-path} also shows OGD-FW degrading more rapidly as $\mathrm{RS}_T$ increases, with plotted crossover thresholds following the prediction $\mathrm{RS}_T^\star = d/\log N$.
{The appendix crossover plot, \Cref{fig:shortest-path-full} in \Cref{app:shortest-path}, gives a secondary ratio-based summary of the same shortest-path experiment.}
{Thus the main shortest-path figure provides evidence that the same fixed-partition instability model can be informative beyond Lov\'{a}sz games.}

\begin{figure}[t]
\centering
\includegraphics[width=\textwidth]{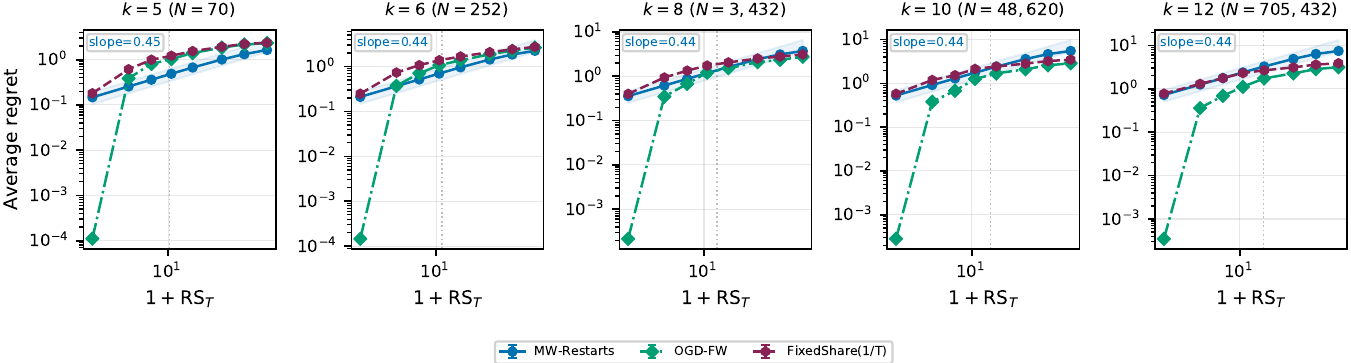}
\vspace{-0.5em}
\caption{Online shortest-path on grid DAGs ($T = 50{,}000$, 10 seeds).
Five grid sizes from $k=5$ ($N=70$) to $k=12$ ($N=705{,}432$).
MW-with-restarts (blue) tracks the theoretical $C\sqrt{(1{+}\mathrm{RS}_T)\log N / T}$ rate (light blue band) with log-log slope $\approx 0.44$; OGD-FW (green) degrades with $\mathrm{RS}_T$.
Crossover threshold $\mathrm{RS}_T^\star = d / \log N$ shifts rightward for smaller~$k$ (dotted vertical lines).}
\label{fig:shortest-path}
\end{figure}

\paragraph{{Finding 4: the tested SNAP influence-maximization instances show low observed instability.}}
{As shown in \Cref{fig:large-influence} in \Cref{app:snap-extended}, the tested SNAP influence-maximization instances remain in a low-instability regime.}
\Cref{fig:large-influence} shows the observed best-response switch counts staying small, with $\widehat{\mathrm{SC}}_T \in \{4,5,6\}$, which is less than $0.3\%$ of the horizon.
The same figure shows that when $T/n \approx 1$, OLMDA is often the strongest baseline even though the switch counts remain essentially fixed, indicating that the differences across networks are better explained by the phase transition in $T/n$ than by growing instability.

\paragraph{Finding 5: exploiting overlap between nearby cells improves constants in practice.}
As shown in \Cref{fig:snap-extended} in Appendix \ref{app:snap-extended}, WS-CAMW achieves $2$--$3\times$ lower absolute regret than cold-start CAMW on the extended SNAP experiments.
The figure therefore supports a constant-factor improvement from transferring weight across adjacent cells, rather than a change in the minimax dependence on $\mathrm{SC}_T$.
This makes the gain practical rather than asymptotic, consistent with the geometric motivation behind chain-overlap transfer.

\paragraph{{Finding 6: revision diagnostics support the warm-start and tracking mechanisms.}}
{Appendix \ref{app:revision-diagnostics} reports targeted checks not shown in main-paper figures. In the transfer-floor sweep, cold CAMW has mean regret $0.0633$, WS-CAMW improves to $0.0600$--$0.0610$, and Geometric-WS-CAMW improves to $0.0440$--$0.0528$. The SAOL baseline was not competitive in tested Lov\'{a}sz regimes. Noise stress tests show that warm-start variants absorb oracle noise much better than cold CAMW, while deterministic switch-detection validation shows near-perfect event detection in the tested regime. These results are used as mechanism diagnostics rather than new headline empirical claims.}

\paragraph{{Additional results.}}
{Appendix \ref{app:experiments} contains additional experimental results, including full hyperparameter and compute details, revision-time diagnostics for warm starts and switch tracking, additional real-data domains for sensor placement and feature selection, ablations of restart rules, tracking variants, step sizes, zeroth-order baselines, and wall-clock scaling, plus extended shortest-path, SNAP, and dimension-scaling plots.}

\section{Conclusion}\label{sec:conclusion}


This paper studied \emph{polyhedral instability}: the number of changes of the active polyhedral region as a structural parameter governing regret in online learning with piecewise linear objectives. 
For fixed finite polyhedral partitions under full-information and bounded-loss assumptions, we established the rate $\Regret_T = \Theta(\sqrt{(1{+}\mathrm{RS}_T)\,T\,\log V_{\max}})$ with a phase transition at $\mathrm{RS}_T^\star = d/\log V_{\max}$. 
For Lov\'{a}sz games, this reduces to the permutation-switch count $\mathrm{SC}_T$, yielding the oracle benchmark $\Regret_T = \Theta(\sqrt{(1{+}\mathrm{SC}_T)\,T\,\log n})$ and an implementable boundary-aware algorithm controlled by the observable count $\widehat{\mathrm{SC}}_T$. 
Our results show that regret is governed by active-region instability rather than ambient dimension alone: stable instances admit experts-like rates, while frequent region changes recover classical dimension-dependent behavior. 
Experiments on synthetic and real combinatorial problems support the predicted scaling and suggest that low-instability regimes can arise in practice.

\section*{Limitations}

This work studies regret through the lens of polyhedral instability under fixed-partition and full-information assumptions. These assumptions make the geometry transparent, but they also limit the direct applicability of the theory to settings with changing partitions, partial feedback, noisy observations, or approximate optimization oracles. The empirical results support the predicted scaling in controlled and preliminary real-data settings, but they should not be interpreted as evidence that low-instability behavior is universal across all combinatorial online-learning problems. A further limitation is that the practical algorithms rely on detecting or approximating active-region changes; in large-scale deployments, this detection step may introduce computational overhead or additional approximation error.

\section*{Ethical Statement}

This paper is primarily theoretical and methodological. It does not introduce a deployed system, collect new human-subject data, or make decisions about individuals that may raise ethical concerns.

\section*{Broader Impact}

By identifying polyhedral instability as a predictor of regret, this work may help make large-scale online combinatorial decision problems more computationally tractable and interpretable. The results suggest that some problems with exponentially many actions can behave like much smaller learning problems when their active geometric structure is stable. This could benefit applications such as network routing, monitoring, resource allocation, and influence modeling. At the same time, improved online decision methods can also be used in high-stakes or manipulative settings, especially when applied to recommendation, targeting, or information-spread problems. The broader impact therefore depends on careful application-specific governance, transparent evaluation, and safeguards against harmful optimization objectives.

\section*{LLM Usage}
We used large language models (LLMs) to support the preparation of this manuscript, including assistance with writing, editing, and improving clarity of presentation. LLMs were also used to support experimental workflows, such as drafting code, checking implementation details, and organizing analysis. All intellectual contributions, methodological decisions, experimental design, results interpretation, and final manuscript content were reviewed and verified by the authors, who take full responsibility for the accuracy and integrity of the work.

\FloatBarrier
\newpage

\bibliographystyle{plainnat}
\bibliography{references_online_minmax}

\newpage
\appendix

\vspace{0.5em}
\begin{center}
\textbf{\large Appendix Table of Contents}
\end{center}
\vspace{0.3em}
\noindent
\renewcommand{\arraystretch}{1.1}
\begin{tabular}{@{}l@{\quad}p{\dimexpr\textwidth-5em\relax}@{}}
\textbf{A.} & \nameref{app:background-foundations} \leaders\hbox to.5em{\hss.\hss}\hfill \pageref{app:background-foundations} \\[0.4ex]
\textbf{B.} & \nameref{app:algorithms} \leaders\hbox to.5em{\hss.\hss}\hfill \pageref{app:algorithms} \\[0.4ex]
\textbf{C.} & \nameref{app:upper-bounds} \leaders\hbox to.5em{\hss.\hss}\hfill \pageref{app:upper-bounds} \\[0.4ex]
\textbf{D.} & \nameref{app:lower-bounds} \leaders\hbox to.5em{\hss.\hss}\hfill \pageref{app:lower-bounds} \\[0.4ex]
\textbf{E.} & \nameref{app:separations} \leaders\hbox to.5em{\hss.\hss}\hfill \pageref{app:separations} \\[0.4ex]
\textbf{F.} & \nameref{app:experiments} \leaders\hbox to.5em{\hss.\hss}\hfill \pageref{app:experiments} \\[0.4ex]
\textbf{G.} & \nameref{app:limitations-related} \leaders\hbox to.5em{\hss.\hss}\hfill \pageref{app:limitations-related} \\[0.4ex]
\end{tabular}
\vspace{1em}


\section{Background and Structural Foundations}\label{app:background-foundations}

\subsection{Landscape of Online Polyhedral Learning}\label{app:extended-intro}

Table~\ref{tab:comparison} summarizes the landscape of online submodular and polyhedral optimization.
{Our work provides an instability-parameterized characterization for the fixed-partition setting studied in the main text.}

\begin{table}[htbp]
\centering
\small
\caption{Comparison of online learning approaches for submodular-concave and polyhedral games.}
\label{tab:comparison}
\begin{tabular}{@{}llll@{}}
\toprule
\textbf{Approach} & \textbf{Rate} & \textbf{Adaptive?} & \textbf{Tight?} \\
\midrule
OGD / Mirror Descent & $O(\sqrt{n/T})$ & No & Yes (worst case) \\
ZO-EG~\citep{FarzinEtAl2025} & $O(\sqrt{n\bar{P}_T/T})$ & Path-length & No \\
Naive MW ($2^n$ arms) & $O(\sqrt{n/T})$ & No & Intractable \\
CAMW (this paper) & $O(\sqrt{(1{+}\mathrm{SC}_T)\log n / T})$ & Cell stability & Yes \\
\bottomrule
\end{tabular}
\end{table}

\subsection{Background on Structural Properties of the Lov\'{a}sz Extension}\label{app:background}

We present the following structural results from~\citet{Lovasz1983, Schrijver2003} for completeness.

\begin{lemma}[Lov\'{a}sz extension: convexity and tight relaxation]\label{lem:lovasz-convexity}
Let $f \colon 2^{[n]} \to \R$ be submodular.
Then the Lov\'{a}sz extension $\fL \colon [0,1]^n \to \R$ defined in~\eqref{eq:lovasz} satisfies: (1) $\fL$ is convex and (2) $\min_{x \in [0,1]^n} \fL(x) = \min_{S \subseteq [n]} f(S)$.
\end{lemma}



\begin{lemma}[Mixed minimax value via threshold coupling]\label{lem:mixed-value}
Consider a submodular-concave game $(f, [n], \Y)$ with $f(\cdot, y)$ submodular and $f(S, \cdot)$ concave.
Define the \emph{mixed value} $\Vmix := \min_{\mu \in \Delta(2^{[n]})} \max_{y \in \Y} \E_{S \sim \mu}[f(S,y)]$ and the \emph{relaxed value} $\Vstar := \min_{x \in [0,1]^n} \max_{y \in \Y} \fL(x,y)$.
Then $\Vstar = \Vmix$.
\end{lemma}



\begin{lemma}[Chain support structure]\label{lem:chain-support}
The optimal mixed strategy for the minimizer in a submodular-concave game is supported on a chain of at most $n{+}1$ nested subsets:
there exist $\emptyset = C_0 \subset C_1 \subset \cdots \subset C_K = [n]$ with $K \le n$ and weights $w \in \Delta_{K+1}$ such that
$\fL(x^\star, y) = \sum_{j=0}^{K} w_j \, f(C_j, y)$ for all $y \in \Y$.
\end{lemma}


\begin{corollary}[Mixed equilibrium existence]\label{cor:mixed-eq}
Every submodular-concave game admits a saddle point $(x^\star, y^\star) \in [0,1]^n \times \Y$ with $\fL(x^\star, y) \le \Vstar \le \fL(x, y^\star)$ for all $x \in [0,1]^n$, $y \in \Y$.
Moreover, the optimal minimizer strategy $x^\star$ decomposes as a distribution over a chain of at most $n{+}1$ nested subsets, and $\Vstar = \Vmix$.
\end{corollary}


\subsection{Assumption Scaling Across Domains}\label{app:assumptions}

Assumption \ref{assump:bounded} requires Lipschitz and boundedness constants $(L_x, L_y, M)$.
\Cref{tab:assumptions} verifies that these scale polynomially in~$n$ for all application domains considered in the paper.
The key observation is that the $\sqrt{n}$ factor in the OLMDA static regret arises from the diameter of $[0,1]^n$, not from these constants; CAMW replaces this with $\sqrt{\log n}$ by exploiting the chain structure.

\begin{table}[htbp]
\centering
\small
\caption{Boundedness constants scale polynomially in~$n$ for standard application domains.}
\label{tab:assumptions}
\begin{tabular}{@{}llll@{}}
\toprule
\textbf{Domain} & $L_x$ & $L_y$ & $M$ \\
\midrule
Influence max.\ (IC model) & $O(\sqrt{n})$ & $O(n)$ & $O(n)$ \\
Sensor placement (coverage) & $O(\sqrt{n})$ & $O(n)$ & $O(n)$ \\
Feature selection (MI) & $O(\sqrt{n})$ & $O(n \log n)$ & $O(n)$ \\
\bottomrule
\end{tabular}
\end{table}

\section{Algorithms and Implementation}\label{app:algorithms}

\subsection{OLMDA: Online Lov\'{a}sz Mirror Descent-Ascent}\label{app:olmda-details}

\begin{algorithm}[htbp]
\caption{OLMDA: Online Lov\'{a}sz Mirror Descent-Ascent}
\label{alg:olmda}
\begin{algorithmic}[1]
\REQUIRE Step sizes $\eta_x, \eta_y > 0$; initial $(x_1, y_1)$
\FOR{$t = 1, \dots, T$}
  \STATE \textbf{Play} $(x_t, y_t)$; \textbf{observe} $f_t$
  \STATE $x_{t+1} \gets \argmin_{x \in [0,1]^n} \inner{g_t^x}{x} + \frac{1}{\eta_x} \Breg(x, x_t)$ \hfill [mirror descent]
  \STATE $y_{t+1} \gets \argmax_{y \in \Y} \inner{g_t^y}{y} - \frac{1}{\eta_y} D_\varphi(y, y_t)$ \hfill [mirror ascent]
\ENDFOR
\end{algorithmic}
\end{algorithm}

Baseline algorithm OLMDA treats $\fL$ as a generic convex-concave function.
The Lov\'{a}sz subgradient at~$x_t$ requires $n{+}1$ evaluations of~$f_t$ and $O(n \log n)$ sorting.

\begin{theorem}[Static regret]\label{thm:static-regret}
Under Assumption \ref{assump:bounded}, OLMDA achieves $\frac{1}{T} \Regret_T \le (L_x \sqrt{n} + L_y D_\Y)/\sqrt{T}$.
\end{theorem}

The $\sqrt{n}$ factor arises from the diameter of~$[0,1]^n$ and treats the polyhedral structure as irrelevant.

\begin{theorem}[Dynamic regret]\label{thm:dynamic-regret}
Under Assumption \ref{assump:bounded}, \Cref{alg:olmda} with step sizes $\eta_x = \sqrt{n(1+P_T^x)}/(L_x\sqrt{T})$ achieves:
\begin{equation}
  \frac{1}{T} \DualGap_T \le O\!\left(\sqrt{\frac{n(1 + \bar{P}_T)}{T}}\right),
\end{equation}
where $\bar{P}_T = \sum_{t=1}^{T-1} \norm{z_t^\star - z_{t+1}^\star}$ is the path length of the saddle-point sequence and $z_t^\star = (x_t^\star, y_t^\star)$.
\end{theorem}


\begin{proof}
The proof follows the dynamic regret analysis for mirror descent.
Using the same decomposition as \Cref{thm:static-regret} but with time-varying comparators~$x_t^\star$, the telescoping Bregman terms yield an additional path-length cost:
\begin{equation}
  \sum_{t=1}^{T} \inner{g_t^x}{x_t - x_t^\star} \le \frac{1}{\eta_x}\!\left(\Breg(x_1^\star, x_1) + \sum_{t=1}^{T-1} \Breg(x_{t+1}^\star, x_t^\star)\right) + \frac{\eta_x T L_x^2}{2}.
\end{equation}
The Bregman divergence terms are bounded by $O(n(1 + P_T^x))$ under the Euclidean setup, giving the result after optimizing~$\eta_x$.
\end{proof}

\subsection{Geometric Chain-Overlap Transfer}\label{app:warmstart}

Geometric WS-CAMW (\Cref{alg:camw}) replaces the uniform restart of cold-start CAMW with a chain-overlap transfer that preserves weights on shared chain elements.
We state and prove the initialization guarantee used in \Cref{thm:implementable-camw}.

\begin{lemma}[Geometric transfer initialization]\label{lem:geometric-transfer}
Let $w$ be the MW weight distribution at the end of epoch~$j$ over chain~$\mathcal{C}^{\pi_j}$, and let $\mathcal{S} = \{k : C_k^{\pi_j} = C_k^{\pi_{j+1}}\}$ be the shared chain elements with total weight $p = \sum_{k \in \mathcal{S}} w_k$.
Under geometric transfer with parameter~$\alpha \in (0,1)$, the transferred distribution
\begin{equation}\label{eq:transfer-rule}
  \tilde{w}_k =
  \begin{cases}
    (1{-}\alpha)\,w_k + \alpha/(n{+}1) & \text{if } k \in \mathcal{S}, \\
    \alpha/(n{+}1) & \text{if } k \notin \mathcal{S},
  \end{cases}
  \qquad w'_k = \tilde{w}_k \Big/ \sum_\ell \tilde{w}_\ell,
\end{equation}
satisfies:
\begin{enumerate}
  \item[\textup{(a)}] $w'_k \ge \alpha/(n{+}1)$ for every $k$, so $D_{\mathrm{KL}}(e_{i^\star} \| w') \le \log((n{+}1)/\alpha)$ for any comparator~$i^\star$.
  \item[\textup{(b)}] If $i^\star \in \mathcal{S}$ with prior weight $w_{i^\star}$, then $w'_{i^\star} \ge (1{-}\alpha)\,w_{i^\star} + \alpha/(n{+}1)$, and the initialization KL is
  \[
    D_{\mathrm{KL}}(e_{i^\star} \| w') \le \log\!\bigl(1/\bigl((1{-}\alpha)\,w_{i^\star} + \alpha/(n{+}1)\bigr)\bigr).
  \]
  When $w_{i^\star}$ is large (the expert performed well in the previous epoch), this can be $O(1)$ rather than~$O(\log n)$.
\end{enumerate}
\end{lemma}

\begin{proof}
\textbf{Part (a).}
The normalization constant is $Z = \sum_k \tilde{w}_k = (1{-}\alpha)\,p + \alpha \le 1$, since $p \le 1$.
Thus $w'_k = \tilde{w}_k / Z \ge \tilde{w}_k \ge \alpha/(n{+}1)$ for all~$k$.
The KL bound follows: $D_{\mathrm{KL}}(e_{i^\star} \| w') = \log(1/w'_{i^\star}) \le \log((n{+}1)/\alpha)$.

\textbf{Part (b).}
For $i^\star \in \mathcal{S}$: $\tilde{w}_{i^\star} = (1{-}\alpha)\,w_{i^\star} + \alpha/(n{+}1)$.
Since $Z \le 1$, $w'_{i^\star} = \tilde{w}_{i^\star}/Z \ge \tilde{w}_{i^\star}$, giving the stated bound.
When $p \approx 1$ (most weight on shared elements, as occurs after small permutation changes), $Z \approx 1$ and $w'_{i^\star} \approx (1{-}\alpha)\,w_{i^\star}$, so $D_{\mathrm{KL}} \approx \log(1/w_{i^\star}) + \log(1/(1{-}\alpha))$: the inherited divergence from the previous epoch plus a constant.
\end{proof}

\begin{remark}[Comparison with Jaccard transfer]
An alternative transfer uses the Jaccard-overlap matrix $T_{ik} = |C_i^{(j+1)} \cap C_k^{(j)}| / |C_i^{(j+1)} \cup C_k^{(j)}|$ and sets $w_i' = (1{-}\alpha)\sum_k T_{ik} w_k^{(j)} + \alpha/(n{+}1)$.
This satisfies the same worst-case bound (Part~(a)) but does not exploit the position-based chain structure: the geometric transfer preserves the \emph{exact} weight of shared elements, while Jaccard diffuses weight through set-overlap coefficients.
In experiments, both variants outperform cold restart, with geometric transfer providing slightly sharper gains when consecutive permutations are close.
\end{remark}

Experiments show $6$--$67\%$ improvement in regret over cold-start CAMW across domains (\Cref{app:experiments}), with the largest gains at moderate $T/n$ where burn-in dominates.
At larger scale, the improvement is more dramatic: in the extended SNAP experiments (\Cref{app:snap-extended}), Geometric WS-CAMW achieves $2$--$3\times$ lower \emph{absolute} regret than cold-start CAMW throughout the $T/n$ range, with the gap widening as $T$ grows.

\section{Proofs: Upper Bound}\label{app:upper-bounds}

\subsection{Preliminary Results}\label{app:core-lemmas}

The following two lemmas encapsulate the key steps of the CAMW analysis.
They are used throughout the upper bound proofs (Theorems~\ref{thm:curvature-regret}, \ref{thm:implementable-camw}, \ref{thm:chain-regret})
).

\begin{lemma}[Within-epoch MW bound]\label{lem:epoch-mw}
Fix an epoch of length~$\tau$ during which the minimizer runs Multiplicative Weights over a fixed chain of $K \le n{+}1$ nested subsets with losses bounded by~$M$.
Then the epoch's contribution to saddle-point regret against the game value~$V^\star$ satisfies:
\begin{equation}\label{eq:lem-epoch-mw}
  \sum_{t \in \mathrm{epoch}} \fL(x_t, y^\star) - \tau \cdot V^\star \le 2M\sqrt{\tau \ln K}.
\end{equation}
\end{lemma}

\begin{proof}
Standard MW with learning rate $\eta = \sqrt{\ln K / \tau}$ over $K$ experts with losses in $[-M, M]$ yields regret $\le 2M\sqrt{\tau \ln K}$ against the best expert~\citep{FreundSchapire1997}.
Since $V^\star = \min_x \max_y \fL(x,y) \ge \min_i f(C_i, y^\star)$ (the game value exceeds the best chain vertex against any fixed~$y^\star$), the MW regret against the best vertex upper bounds the epoch's contribution to saddle-point regret.
\end{proof}

\begin{lemma}[Epoch aggregation via Cauchy--Schwarz]\label{lem:epoch-aggregation}
Suppose $E$ epochs of lengths $\tau_1, \ldots, \tau_E$ with $\sum_j \tau_j = T$ each contribute regret $a\sqrt{\tau_j}$ for some constant~$a > 0$.
Then the total regret satisfies:
\begin{equation}\label{eq:lem-cs}
  \sum_{j=1}^{E} a\sqrt{\tau_j} \le a\sqrt{E \cdot T}.
\end{equation}
\end{lemma}

\begin{proof}
By Cauchy--Schwarz: $\sum_{j=1}^E \sqrt{\tau_j} \le \sqrt{E \cdot \sum_j \tau_j} = \sqrt{E \cdot T}$.
\end{proof}

\subsection[Proof of the curvature-regret theorem]{Proof of \Cref{thm:curvature-regret}}\label{app:proof-curvature}

We prove the bound for the oracle version of CAMW, which assumes access to the equilibrium permutation~$\pi_{x_t^\star}$.
This characterizes minimax regret as a function of the intrinsic permutation-switch complexity~$\mathrm{SC}_T$.

\paragraph{Proof strategy.}
The proof proceeds in four steps.
First, we decompose the horizon into $\mathrm{SC}_T + 1$ epochs separated by permutation switches (\textbf{epoch decomposition}).
Second, within each epoch the optimal chain is fixed, so we run MW over its $\le n{+}1$ vertices and apply Lemma \ref{lem:epoch-mw} (\textbf{within-epoch structure}).
Third, we aggregate across epochs via Lemma \ref{lem:epoch-aggregation} (\textbf{Cauchy--Schwarz aggregation}).
Fourth, we combine with the maximizer's mirror ascent bound.
The key insight is that MW's $\log K$ dependence on the number of chain vertices replaces OGD's $n$-dependence on the ambient dimension.

\begin{proof}

\textbf{Epoch decomposition.}\enspace
Let $\{t_1 = 1 < t_2 < \cdots < t_{S+1} = T{+}1\}$ denote the time indices at which the equilibrium permutation changes, using the consistent deterministic selection rule from Lemma \ref{def:switch-count}.
This partitions $\{1, \dots, T\}$ into $S{+}1$ epochs, where $S = \mathrm{SC}_T$.

\medskip\noindent\textbf{Within-epoch structure.}\enspace
Fix an epoch~$j$ of length~$T_j$.
During this epoch, the equilibrium permutation~$\pi^{(j)}$ is fixed.
The Lov\'{a}sz extension restricted to this permutahedron cell reduces to a linear function over the associated chain
\[
  C_0^{(j)} \subset C_1^{(j)} \subset \cdots \subset C_n^{(j)}.
\]
Thus the minimizer's problem reduces to online linear optimization over the $n{+}1$ vertices $\{C_i^{(j)}\}_{i=0}^n$.

\medskip\noindent\textbf{Minimizer regret within an epoch.}\enspace
By Lemma \ref{lem:epoch-mw} with $K = n{+}1$ and losses bounded by~$M$:
\begin{equation}\label{eq:epoch-mw}
  \sum_{t \in \text{epoch } j} \fL(x_t, y^\star) - T_j \cdot V^\star \le 2M\sqrt{T_j \ln(n{+}1)}.
\end{equation}

\medskip\noindent\textbf{Coupling with the maximizer.}\enspace
The above bounds minimizer regret against a fixed comparator within the epoch.
To obtain saddle-point regret as defined in~\eqref{eq:sp-regret}, we combine this with mirror ascent for the maximizer.

By standard mirror ascent analysis with step size $\eta_y = 1/\sqrt{T}$:
\begin{equation}\label{eq:max-epoch}
  \sum_{t=1}^T \fL(x_t, y_t) - \sum_{t=1}^T \fL(x_t, y_t^\star) \le O(L_y D_\Y \sqrt{T}).
\end{equation}

Combining the minimizer and maximizer bounds and using the minimax inequality $\Vstar = \min_x \max_y \fL(x, y) \ge \min_x \fL(x, y)$ for every fixed~$y$, we obtain total regret
\begin{equation}\label{eq:min-epoch-total}
  \Regret_T \le \sum_{j=1}^{S+1} 2M\sqrt{T_j \ln(n{+}1)} + O(L_y D_\Y \sqrt{T}).
\end{equation}

\medskip\noindent\textbf{Aggregation across epochs.}\enspace
By Lemma \ref{lem:epoch-aggregation} with $E = S{+}1 = \mathrm{SC}_T{+}1$ epochs:
\begin{equation}
  \Regret_T \le O\!\left(M\sqrt{(1{+}\mathrm{SC}_T) \cdot T \cdot \ln n} + L_y D_\Y \sqrt{T}\right).
\end{equation}
Dividing by~$T$:
\begin{equation}
  \frac{1}{T}\Regret_T \le O\!\left(\sqrt{\frac{(1{+}\mathrm{SC}_T) \cdot \log n}{T}}\right) + \frac{L_y D_\Y}{\sqrt{T}}.
\end{equation}
\end{proof}

\subsection[Proof of the implementable CAMW theorem]{Proof of \Cref{thm:implementable-camw}}\label{app:proof-implementable}

\begin{proof}
The proof mirrors that of \Cref{thm:curvature-regret}, with the epoch decomposition now driven by the observed best-response switches and the geometric chain-overlap transfer replacing the uniform restart.

\textbf{Step 1: Epoch decomposition via best-response switches.}
The observed switches $\widehat{\mathrm{SC}}_T$ partition $[T]$ into $\widehat{S} := \widehat{\mathrm{SC}}_T + 1$ epochs.
Within epoch~$j$, the best-response permutation $\pi_t^{\mathrm{BR}}$ is constant, so \Cref{alg:camw} runs MW on a fixed chain $\mathcal{C}^{(j)}$ of $n{+}1$ elements.

\textbf{Comparator.}\enspace
Within each epoch, we bound the minimizer's contribution to saddle-point regret against the game value~$\Vstar$.
Since $\min_i f(C_i^{(j)}, y^\star) \le \Vstar$ for any chain and any~$y^\star$, the MW bound against the best fixed chain vertex in hindsight yields a valid upper bound on the epoch's regret contribution.

\textbf{Step 2: Within-epoch MW regret.}
At the start of epoch~$j$, the geometric transfer (Lemma \ref{lem:geometric-transfer}(a)) ensures every expert has weight $\ge \alpha/(n{+}1)$, so the initialization KL satisfies $D_{\mathrm{KL}}(e_{i^\star} \| w_j') \le \log((n{+}1)/\alpha)$.
Standard MW analysis~\citep{FreundSchapire1997} then gives per-epoch regret:
\begin{equation}
  \sum_{t \in \text{epoch } j} \fL(x_t, y^\star) - T_j \cdot \Vstar \le 2M\sqrt{T_j \ln((n{+}1)/\alpha)}.
\end{equation}
Since $\alpha$ is a fixed constant, $\ln((n{+}1)/\alpha) = \ln(n{+}1) + \ln(1/\alpha) = O(\ln n)$.

\textbf{Step 3: Cauchy--Schwarz aggregation.}
By Lemma \ref{lem:epoch-aggregation} with $E = \widehat{\mathrm{SC}}_T{+}1$:
\begin{equation}
  \sum_{t=1}^T \fL(x_t, y^\star) - T \cdot \Vstar \le 2M\sqrt{(\widehat{\mathrm{SC}}_T{+}1) \cdot T \cdot \ln((n{+}1)/\alpha)}.
\end{equation}

\textbf{Step 4: Combining with maximizer.}
The maximizer's mirror ascent contributes $L_y D_\Y \sqrt{T}$ independently of epoch structure.
Combining: $\frac{1}{T}\Regret_T \le O\!\left(\sqrt{(1{+}\widehat{\mathrm{SC}}_T) \cdot \log n / T}\right) + L_y D_\Y / \sqrt{T}$.

\textbf{Data-dependent refinement.}
By Lemma \ref{lem:geometric-transfer}(b), when the epoch-$j$ comparator is a shared chain element with inherited weight~$w_{i^\star}$, the per-epoch bound improves to $2M\sqrt{T_j \cdot \log(1/w'_{i^\star_j})}$, which can be much smaller than $2M\sqrt{T_j \ln(n{+}1)}$.
This does not change the worst-case rate---the adversary can always arrange for the comparator to be a new element with minimal weight---but it explains the consistent empirical gains of geometric transfer on stable instances where consecutive permutations share most of their chain structure.
\end{proof}

\subsection{Chain-Restricted Regret}\label{app:proof-chain}

\begin{theorem}[Chain-restricted regret]\label{thm:chain-regret}
If the game admits an optimal chain of at most $K \le n+1$ sets at each round, MW over the chain achieves $\frac{1}{T}\Regret_T = O(\sqrt{\log K / T})$.
When the optimal chain varies across rounds with $\mathrm{SC}_T$ permutation switches, sleeping-experts MW over the union chain pool achieves $\frac{1}{T}\Regret_T = O(\sqrt{(1{+}\mathrm{SC}_T)(\log n + \log(1{+}\mathrm{SC}_T))/T})$.
\end{theorem}

\begin{proof}
We give the full proof for the single-game repeated setting ($f_t = f$ for all~$t$) and then discuss the extension to varying games.

\textbf{Step 1: Chain structure of optimal strategies.}
The optimal mixed strategy $x^\star \in [0,1]^n$ for $\min_x \max_y \fL(x,y)$ is a threshold vector supported on an optimal chain $\mathcal{C}^\star = \{C_0 = \emptyset \subset C_1 \subset \cdots \subset C_K\}$ with $K \le n+1$ nested subsets (Lemma \ref{lem:chain-support}).\label{prop:support}
Concretely, there exist weights $w^\star \in \Delta_{K}$ (the probability simplex over $K{+}1$ elements) such that
\begin{equation}\label{eq:chain-decomp}
  \fL(x^\star, y) = \sum_{j=0}^{K} w_j^\star \, f(C_j, y), \qquad \text{for all } y \in \Y.
\end{equation}
Let $V^\star = \min_x \max_y \fL(x,y) = \fL(x^\star, y^\star)$ denote the game value, where $y^\star = \argmax_y \fL(x^\star, y)$.

\textbf{Step 2: Two-player algorithm.}
We run the following pair of no-regret algorithms:
\begin{itemize}
  \item \emph{Minimizer}: Multiplicative Weights (MW) over the $K{+}1$ experts $\{C_0, C_1, \ldots, C_K\}$.
  At round~$t$, MW produces a distribution $p_t \in \Delta_K$ over the chain.
  The learner plays the threshold vector~$x_t$ whose Lov\'{a}sz decomposition on $\mathcal{C}^\star$ has weights~$p_t$, so that $\fL(x_t, y) = \sum_{j=0}^K p_{t,j}\, f(C_j, y)$ for all~$y$.
  \item \emph{Maximizer}: Online mirror ascent over~$\Y$ with Bregman divergence~$D_\varphi$ and step size $\eta_y = D_\Y/(L_y\sqrt{T})$.
\end{itemize}

\textbf{Step 3: Minimizer's regret via MW.}
The minimizer's loss at round~$t$ against the fixed saddle-point response~$y^\star$ is $\fL(x_t, y^\star) = \sum_j p_{t,j}\, f(C_j, y^\star)$.
Since $|f(C_j, y^\star)| \le M$ for all~$j$ by Assumption \ref{assump:bounded}, the standard MW guarantee~\citep{AbernethyEtAl2008} yields:
\begin{equation}\label{eq:mw-chain}
  \sum_{t=1}^T \fL(x_t, y^\star) - T \cdot \min_{j \in \{0,\ldots,K\}} f(C_j, y^\star) \le 2M\sqrt{T \ln(K{+}1)}.
\end{equation}
Since $V^\star = \sum_j w_j^\star f(C_j, y^\star)$ is a convex combination, we have $V^\star \ge \min_j f(C_j, y^\star)$, so:
\begin{equation}\label{eq:min-chain-bound}
  \sum_{t=1}^T \fL(x_t, y^\star) - T \cdot V^\star \le 2M\sqrt{T \ln(K{+}1)}.
\end{equation}

\textbf{Step 4: Combining with maximizer's mirror ascent.}
Mirror ascent gives $T V^\star - \sum_{t=1}^T \fL(x^\star, y_t) \le L_y D_\Y \sqrt{T}$.
Adding to~\eqref{eq:min-chain-bound}:
\begin{equation}
  \Regret_T \le 2M\sqrt{T \ln(K{+}1)} + L_y D_\Y \sqrt{T},
\end{equation}
so $\frac{1}{T}\Regret_T = O(\sqrt{\log K / T})$, which gives $O(\sqrt{\log n / T})$ for $K \le n+1$.

\textbf{Step 5: Extension to varying games ($f_t$ changes across rounds).}
When the optimal chain $\mathcal{C}_t^\star$ varies across rounds, we proceed as follows.
Let $\mathcal{C}^{(1)}, \ldots, \mathcal{C}^{(\mathrm{SC}_T+1)}$ be the distinct optimal chains encountered across the $\mathrm{SC}_T + 1$ epochs (an epoch is a maximal interval during which the optimal chain is fixed).
Define the \emph{union expert pool} $\mathcal{E} = \bigcup_{j=1}^{\mathrm{SC}_T+1} \mathcal{C}^{(j)}$, which contains at most $(\mathrm{SC}_T + 1)(n+1)$ nested-set experts.
We run MW in \emph{sleeping-experts mode}~\citep{AbernethyEtAl2008}: at each round~$t$, only the experts in the current chain $\mathcal{C}_t^\star$ are ``awake'' (receive loss feedback and participate in the MW update), while experts from other chains ``sleep'' (their weights are frozen).

The sleeping-experts guarantee~\citep[Theorem~3]{AbernethyEtAl2008} ensures that for any fixed comparator expert $e \in \mathcal{E}$ and any subsequence of rounds $\mathcal{T}_e \subseteq [T]$ during which $e$ is awake:
\begin{equation}
  \sum_{t \in \mathcal{T}_e} \ell_t(i_t) - \sum_{t \in \mathcal{T}_e} \ell_t(e) \le 2M\sqrt{|\mathcal{T}_e| \cdot \ln |\mathcal{E}|},
\end{equation}
where $\ell_t(i_t)$ is the loss of the action chosen by MW.

Within each epoch~$j$ of length~$T_j$, the optimal chain $\mathcal{C}^{(j)}$ is fixed, and the best expert on this chain achieves value $\le V_j^\star$.
The sleeping-experts bound restricted to epoch~$j$ gives minimizer regret $O(M\sqrt{T_j \ln |\mathcal{E}|})$.
Aggregating over epochs via Cauchy--Schwarz:
\begin{equation}
  \sum_{j=1}^{\mathrm{SC}_T+1} M\sqrt{T_j \ln |\mathcal{E}|}
  \le M\sqrt{\ln |\mathcal{E}|} \cdot \sqrt{(\mathrm{SC}_T + 1) \sum_j T_j}
  = M\sqrt{(\mathrm{SC}_T + 1)\, T\, \ln |\mathcal{E}|}.
\end{equation}
Since $|\mathcal{E}| \le (\mathrm{SC}_T + 1)(n+1)$, we have $\ln |\mathcal{E}| \le \ln(\mathrm{SC}_T + 1) + \ln(n+1)$.
Combined with the maximizer's mirror ascent bound (unchanged), this yields:
\begin{equation}
  \frac{1}{T}\Regret_T = O\!\left(\sqrt{\frac{(1{+}\mathrm{SC}_T)(\ln n + \ln(1{+}\mathrm{SC}_T))}{T}}\right),
\end{equation}
which is $O(\sqrt{(1{+}\mathrm{SC}_T)\log n / T})$ whenever $\mathrm{SC}_T = \mathrm{poly}(n, T)$ (the logarithmic $\ln(1{+}\mathrm{SC}_T)$ term is absorbed).
\end{proof}

\section{Proofs: Lower Bounds and Tight Characterization}\label{app:lower-bounds}

\subsection{Structure-Dependent Lower Bound}\label{app:proof-lower}

\begin{lemma}
    For any $c \in [0,1]$, there exists a family of submodular-concave games such that $\frac{1}{T}\E[\Regret_T] \ge \Omega(\sqrt{(1{-}c) \cdot n / T})$.

\end{lemma}

\begin{proof}
Fix $c \in [0,1]$ and set $m = \lceil (1{-}c)\,n \rceil$.
Partition $[n] = A \cup B$ with $|A| = m$, $|B| = n - m$.
Take $h(T) = \min(|T|, 1)$ for $T \subseteq B$ (monotone submodular) and $\Y = [-1, 1]$.
At each round $t$, draw i.i.d.\ $\varepsilon_{t,i} \sim \mathrm{Unif}\{-1,+1\}$ for $i \in A$ and define
\begin{equation}\label{eq:lower-construction}
  f_t(S, y) = y \sum_{i \in S \cap A} \varepsilon_{t,i} \;+\; h(S \cap B), \qquad S \subseteq [n],\; y \in \Y.
\end{equation}
Each $f_t$ is submodular in~$S$ (sum of modular and submodular) and linear in~$y$.

The Lov\'{a}sz extension separates as $f_t^L(x, y) = y \sum_{i \in A} x_i\, \varepsilon_{t,i} + h_L(x_B)$,
where $h_L(x_B)$ is deterministic and $y$-independent.
The $B$-component cancels in the regret difference, so the problem reduces to online linear optimization over $[0,1]^m \times [-1,1]$ with bilinear payoff $\ell_t(x_A, y) = y\, \langle \varepsilon_t, x_A \rangle$.
By the minimax lower bound for Rademacher losses~\citep{AbernethyEtAl2008} and Yao's principle:
\begin{equation}
  \sup_{\{f_t\}} \frac{1}{T}\, \E\!\left[\Regret_T\right] \ge \Omega\!\left(\sqrt{\frac{(1{-}c) \cdot n}{T}}\right).
\end{equation}
As $c$ increases from~$0$ to~$1$, the effective modular dimension shrinks from~$n$ to~$0$, and the bound interpolates from $\Omega(\sqrt{n/T})$ to~$0$.
The lower bound applies to any curvature notion $\kappa$ satisfying $\kappa = 0$ for modular games and $\kappa \to 1$ as the modular fraction vanishes.
\end{proof}

\subsection[Proof of the lower-switch theorem]{Proof of \Cref{thm:lower-switch}}\label{app:proof-lower-switch}

We construct an adversarial sequence with prescribed permutation-switch count~$s$ such that any algorithm incurs
\[
  \E[\Regret_T] \ge \Omega\!\left(\sqrt{(1{+}s) \cdot T \cdot \log n}\right).
\]

\begin{proof}

\textbf{Epoch construction.}\enspace
Partition the horizon into $s{+}1$ epochs of equal length $T_j = \lfloor T/(s{+}1) \rfloor$.
For each epoch~$j$, choose a permutation~$\pi^{(j)}$ such that consecutive permutations differ in their first coordinate ($\pi^{(j)}(1) \ne \pi^{(j+1)}(1)$), ensuring distinct permutahedron cells.

\medskip\noindent\textbf{Loss definition.}\enspace
Within epoch~$j$, let $\{C_i^{(j)}\}_{i=0}^n$ be the chain corresponding to~$\pi^{(j)}$.
For each round~$t$ in epoch~$j$, sample i.i.d.\ Rademacher variables $\varepsilon_{t,i} \in \{\pm 1\}$ and define:
\begin{equation}
  f_t(S, y) = y \cdot \varepsilon_{t, k(S)}, \qquad k(S) = \max\{i : C_i^{(j)} \subseteq S\}.
\end{equation}
This function is modular in~$S$ along the chain and hence submodular.
We add a deterministic perturbation $\varepsilon \cdot i/T$ to the loss of chain element $C_i^{(j)}$ for $\varepsilon = 1/(nT)$, ensuring strict ordering and guaranteeing the equilibrium permutation equals~$\pi^{(j)}$ uniquely within each epoch.

\medskip\noindent\textbf{Epoch-level lower bound.}\enspace
Within epoch~$j$, the minimizer faces an online learning problem over $n{+}1$ experts with i.i.d.\ Rademacher losses.
The learner cannot transfer information across epochs: losses in epoch~$j$ are drawn independently of prior epochs, so the experts problem is fresh each epoch.
By the standard minimax lower bound for expert advice (via Yao's principle;~\citealp{AbernethyEtAl2008}), for any (possibly history-dependent) algorithm:
\begin{equation}
  \E[\Regret_j] \ge c\sqrt{T_j \log(n{+}1)}
\end{equation}
for a universal constant $c > 0$.
Crucially, this bound holds even conditioning on prior epochs, since the losses in epoch~$j$ are drawn independently.

\medskip\noindent\textbf{Aggregation across epochs.}\enspace
Summing across epochs:
\[
  \E[\Regret_T] = \sum_{j=1}^{s+1} \E[\Regret_j] \ge c \sum_{j=1}^{s+1} \sqrt{T_j \log(n{+}1)}.
\]
Since $T_j = \lfloor T/(s{+}1) \rfloor$:
\[
  \sum_{j=1}^{s+1} \sqrt{T_j} \ge \sqrt{(s{+}1) \cdot T}.
\]
Thus:
\[
  \E[\Regret_T] \ge \Omega\!\left(\sqrt{(1{+}s) \cdot T \cdot \log n}\right).
\]
Dividing by~$T$ yields the stated bound $\frac{1}{T}\E[\Regret_T] \ge \Omega(\sqrt{(1{+}s)\log n/T})$.
\end{proof}

\subsection{General Polyhedral Lower Bound}\label{app:proof-poly-lb}

\begin{proposition}[General polyhedral lower bound]\label{prop:poly-lb}
For any $s \in \{0,\dots,T{-}1\}$, there exists a polyhedral partition and piecewise-linear losses with $\mathrm{RS}_T = s$ such that $\frac{1}{T}\E[\Regret_T] \ge \Omega(\sqrt{(1{+}s)\log V_{\max}/T})$.
\end{proposition}

We construct the adversary explicitly.

\begin{proof}
Partition the horizon into $s{+}1$ epochs of lengths
$T_1,\dots,T_{s+1}$.
During each epoch, fix a region $\mathcal{R}_a$ and define losses
$\ell_t(v_i)=M\varepsilon_{t,i}$ for vertices $v_i\in V(\mathcal{R}_a)$,
with $\varepsilon_{t,i}$ i.i.d.\ Rademacher.
Extend $\ell_t$ linearly over $\mathcal{R}_a$ and assign losses $M{+}1$ to all vertices outside $\mathcal{R}_a$, ensuring the equilibrium remains in $\mathcal{R}_a$.
Insert $O(1/T)$ perturbations to guarantee unique region membership.

By Yao's minimax principle, for any (possibly history-dependent)
algorithm,
\[
  \E[\Regret_j]
  \ge
  c\sqrt{T_j \log |V(\mathcal{R}_a)|}
\]
for some constant $c>0$.
Summing over epochs and applying Cauchy--Schwarz gives
\[
  \E[\Regret_T]
  \ge
  \Omega\!\left(
  \sqrt{(1+s)\,T\,\log V_{\max}}
  \right).
\]
Dividing by $T$ yields the stated bound.

Since the comparator class includes the vertex set, the lower bound for expert advice applies to all algorithms over the convex hull --- continuous strategies cannot beat the discrete optimum when losses are linear within each region.
\end{proof}

\subsection[Proof of the polyhedral OCO theorem]{Proof of \Cref{thm:pl-oco}} \label{app:proof-pl-oco}

\begin{proof}
\textbf{Upper bound.}
The argument mirrors \Cref{thm:curvature-regret} with regions replacing permutation cells.
Let $\{t_1 = 1 < t_2 < \cdots < t_{s+1} = T{+}1\}$ be the region-switch times, giving $s = \mathrm{RS}_T$ and $s{+}1$ epochs.
Within epoch~$j$, losses are linear over region~$\mathcal{R}_{a(j)}$, so the learner faces an online linear optimization problem over $|V(\mathcal{R}_{a(j)})| \le V_{\max}$ vertices.
Running MW with learning rate $\eta_j = \sqrt{\ln V_{\max} / T_j}$ yields per-epoch regret $\le 2M\sqrt{T_j \ln V_{\max}}$.
Summing across epochs and applying Cauchy--Schwarz:
\[
  \sum_{j=1}^{s+1} \sqrt{T_j} \le \sqrt{(s{+}1) \cdot T},
\]
giving $\Regret_T \le O(\sqrt{(1{+}\mathrm{RS}_T) \cdot T \cdot \log V_{\max}})$.

\textbf{Lower bound.}
Proposition \ref{prop:poly-lb} provides the matching $\Omega(\sqrt{(1{+}\mathrm{RS}_T) \cdot T \cdot \log V_{\max}})$ lower bound via Rademacher epoch constructions.
\end{proof}

\subsection[Proof of the phase-transition theorem]{Proof of \Cref{thm:phase-transition}}\label{app:proof-phase-transition}

\begin{proof}
The phase transition follows from comparing the polyhedral rate with the OCO baseline.

\textbf{Upper bound.}
By \Cref{thm:pl-oco}, MW-with-restarts achieves $O(\sqrt{(1{+}\mathrm{RS}_T) \log V_{\max} / T})$.
Standard OGD over $\mathcal{X} \subseteq \R^d$ achieves $O(\sqrt{d/T})$ regardless of $\mathrm{RS}_T$.
Taking the better algorithm yields
\[
  \frac{1}{T}\Regret_T^\star(\mathrm{RS}_T)
  \le O\!\left(\min\!\left\{
  \sqrt{\frac{(1{+}\mathrm{RS}_T)\log V_{\max}}{T}},\;
  \sqrt{\frac{d}{T}}
  \right\}\right).
\]
The crossover occurs when $(1{+}\mathrm{RS}_T)\log V_{\max} = d$, i.e., $\mathrm{RS}_T^\star = d/\log V_{\max} - 1 \approx d/\log V_{\max}$.

\textbf{Lower bound.}
For $\mathrm{RS}_T < \mathrm{RS}_T^\star$: Proposition \ref{prop:poly-lb} gives $\Omega(\sqrt{(1{+}\mathrm{RS}_T) \log V_{\max} / T})$, which dominates $\sqrt{d/T}$.
For $\mathrm{RS}_T > \mathrm{RS}_T^\star$: the standard OCO lower bound $\Omega(\sqrt{d/T})$~\citep{AbernethyEtAl2008} applies, since the problem contains online linear optimization over $\R^d$ as a special case.
Thus both branches of the minimum are tight.
\end{proof}

\section{Separations and Structural Properties}\label{app:separations}

\subsection[Switch complexity vs. path length and gradient variation]{$\mathrm{SC}_T$ Is Incomparable with Path Length and Gradient Variation}\label{app:formal-separation}

\begin{table}[htbp]
\centering
\small
\caption{$\mathrm{SC}_T$ captures combinatorial geometry invisible to existing adaptive measures.}
\label{tab:measures}
\begin{tabular}{@{}llll@{}}
\toprule
\textbf{Measure} & \textbf{Defined on} & \textbf{Captures} & \textbf{Typical bound} \\
\midrule
Path length $\bar{P}_T$ & Comparator motion in $\ell_2$ & Drift of best action & $O(\sqrt{T\bar{P}_T})$ \\
Switching count & \# comparator changes & Discrete instability & $O(\sqrt{S \cdot T \log K})$ \\
Variation $V_T$ & $\sum\|f_t - f_{t-1}\|$ & Loss non-stationarity & $O(\sqrt{V_T})$ \\
$\mathrm{SC}_T$ (ours) & Equilibrium cell changes & Geometric stability & $O(\sqrt{(1{+}\mathrm{SC}_T)T\log n})$ \\
\bottomrule
\end{tabular}
\end{table}

\begin{proposition}[$\mathrm{SC}_T$ vs.\ switching count]\label{prop:not-switching}
$\mathrm{SC}_T$ and $S_T := |\{t : a_{t+1}^\star \ne a_t^\star\}|$ are incomparable:
\textup{(a)}~$\mathrm{SC}_T = 0$, $S_T = \Theta(T)$ is possible;
\textup{(b)}~$\mathrm{SC}_T = \Theta(T)$, $S_T = 0$ is possible.
\end{proposition}

\begin{proof}
\textbf{Part (a): $\mathrm{SC}_T = 0$, $S_T = \Theta(T)$.}
Fix a single permutation $\pi = (1, 2, \dots, n)$ with chain $C_0 = \emptyset, C_1 = \{1\}, \dots, C_n = [n]$.
Define losses so that at odd rounds $t$, the optimal chain vertex is $C_1$ (the singleton), and at even rounds, the optimal vertex is $C_2$ (the pair).
Concretely, set $f_t(S, y) = |S \cap A_t| - \beta |S|$ where $A_t$ alternates between $\{1\}$ and $\{1, 2\}$ and $\beta$ is chosen so the optimizer flips.
The permutation $\pi$ is constant (determined by the same marginal ordering), so $\mathrm{SC}_T = 0$.
But the best action changes every round, giving $S_T = T - 1$.

\textbf{Part (b): $\mathrm{SC}_T = \Theta(T)$, $S_T = 0$.}
Let $n = 2$ and $\Y = [0,1]$.
At each round, define $f_t$ such that the singleton $\{1\}$ is uniquely optimal: $f_t(\{1\}, y) < f_t(\{2\}, y) < f_t(\{1,2\}, y) < f_t(\emptyset, y)$ for all~$y$.
The permutation determining the Lov\'{a}sz cell depends on the \emph{ordering} of marginals $\delta_1(y_t^\star)$ vs.\ $\delta_2(y_t^\star)$, which can alternate by choosing $f_t$ so that $\delta_1(y_t^\star) > \delta_2(y_t^\star)$ at odd rounds and $\delta_1(y_t^\star) < \delta_2(y_t^\star)$ at even rounds, while keeping $\{1\}$ optimal in both cases.
Then $S_T = 0$ but $\mathrm{SC}_T = T - 1$.
\end{proof}

\begin{example}[$\mathrm{SC}_T = 0$ but path length $= \Theta(T)$]\label{ex:sc-zero-path-large}
Consider a repeated game where $f_t(\cdot, y)$ has the same optimal permutation~$\pi^\star$ for all~$t$, but the adversary's equilibrium response~$y_t^\star$ oscillates within~$\Y$, causing $\ell_2$ path length $\Theta(T)$.
Since the cell never changes, $\mathrm{SC}_T = 0$, and CAMW achieves $O(\sqrt{\log n / T})$.
Path-length bounds give $O(\sqrt{n T})$---worse by $\sqrt{nT / \log n}$.
\end{example}

\begin{example}[Small path length but $\mathrm{SC}_T = \Theta(T)$]\label{ex:path-small-sc-large}
Adjacent marginals separated by a vanishing gap $\epsilon/T$: tiny perturbations flip the ordering every round, giving $\mathrm{SC}_T = \Theta(T)$ with $\bar{P}_T = O(\epsilon)$.
\end{example}

\begin{proposition}[Formal separation]\label{prop:formal-separation}
No monotone $\Phi$ satisfies $\mathrm{SC}_T \le \Phi(\bar{P}_T)$ for all game sequences, and no monotone $\Psi$ satisfies $\bar{P}_T \le \Psi(\mathrm{SC}_T)$.
The same holds with gradient variation $V_T$ replacing $\bar{P}_T$.
\end{proposition}

\begin{proof}
Example \ref{ex:sc-zero-path-large} shows $\mathrm{SC}_T = 0$ with $\bar{P}_T = \Theta(T)$, violating $\bar{P}_T \le \Psi(0)$ for large $T$.
Example \ref{ex:path-small-sc-large} shows $\mathrm{SC}_T = \Theta(T)$ with $\bar{P}_T = O(\epsilon)$, violating $\mathrm{SC}_T \le \Phi(O(\epsilon))$ for large $T$.
\end{proof}

{$\mathrm{SC}_T$ is also distinct from switching regret: switching regret measures how often the best \emph{fixed action} changes within an exogenous expert class, while $\mathrm{SC}_T$ measures how often the active cell geometry changes through boundary crossings, and is algorithm-independent.

\subsection{Equilibrium Tracking Under Marginal Gaps}\label{app:tracking}

\begin{proposition}[Tracking the equilibrium]\label{prop:tracking}
Let $L_{xy}$ be the Lipschitz constant of marginals and $\delta_{\min} > 0$ the minimum marginal gap at equilibrium.
Then $\widehat{\mathrm{SC}}_T \le \mathrm{SC}_T + |\{t : \|y_t - y_t^\star\| > \delta_{\min}/(2L_{xy})\}|$.
When the maximizer tracks well, $\widehat{\mathrm{SC}}_T = \mathrm{SC}_T$.
\end{proposition}

\begin{proof}
Fix round~$t$ and suppose $\|y_t - y_t^\star\| \le \delta_{\min}/(2L_{xy})$.
By the Lipschitz condition on marginals:
\begin{equation}
  |\delta_i(y_t) - \delta_i(y_t^\star)| \le L_{xy} \|y_t - y_t^\star\| \le \delta_{\min}/2, \qquad \forall\, i \in [n].
\end{equation}
At equilibrium, the sorted marginals have gaps $\ge \delta_{\min}$:
$\delta_{\pi^\star(i+1)}(y_t^\star) - \delta_{\pi^\star(i)}(y_t^\star) \ge \delta_{\min}$ for all~$i$.
Therefore:
\begin{align}
  \delta_{\pi^\star(i+1)}(y_t) - \delta_{\pi^\star(i)}(y_t)
  &\ge \delta_{\pi^\star(i+1)}(y_t^\star) - \delta_{\pi^\star(i)}(y_t^\star) - 2 \cdot \delta_{\min}/2 \\
  &\ge \delta_{\min} - \delta_{\min} = 0.
\end{align}
So the ordering at~$y_t$ agrees with the ordering at~$y_t^\star$: $\pi_t^{\mathrm{BR}}(y_t) = \pi_t^{\mathrm{BR}}(y_t^\star)$.
Since $\pi_t^{\mathrm{BR}}(y_t^\star) = \pi_{x_t^\star}$ (the best-response permutation at the equilibrium maximizer equals the equilibrium minimizer's permutation), the observed permutation $\pi_t^{\mathrm{BR}}(y_t)$ switches only when $\pi_{x_t^\star}$ switches.

For rounds where $\|y_t - y_t^\star\| > \delta_{\min}/(2L_{xy})$, the best-response permutation may differ, contributing at most one extra observed switch per such round.
Thus $\widehat{\mathrm{SC}}_T \le \mathrm{SC}_T + |\{t : \|y_t - y_t^\star\| > \delta_{\min}/(2L_{xy})\}|$.
\end{proof}

For applications with natural marginal gaps (influence: $\delta_{\min} \ge \Omega(1/\mathrm{poly}(n))$; coverage: $\delta_{\min} \ge 1/n$), tracking is effective after $O(L_{xy}^2 D_\Y^2 / \delta_{\min}^2)$ rounds.

\subsection[A priori bound on switch complexity under bounded drift]{A Priori Bound on $\mathrm{SC}_T$ Under Bounded Drift}\label{app:apriori-sc}

\begin{proposition}[A priori bound]\label{prop:apriori-sc}
Under drift $\max_{S,y}|f_{t+1}(S,y) - f_t(S,y)| \le \Delta$ and marginal gap $\delta_{\min} > 0$:
$\mathrm{SC}_T \le 2n\Delta T / \delta_{\min}$.
\end{proposition}

\begin{proof}
Each permutation switch requires adjacent sorted marginals to cross, consuming at least $\delta_{\min}$ of total marginal shift.
Per-round drift shifts at most $n$ marginals by $\Delta$ each, giving budget $n\Delta T$.
\end{proof}

\subsection{Region-Switch Incomparability for General Polyhedral Losses}\label{app:incomparability}

For general polyhedral losses, region-switch complexity $\mathrm{RS}_T$ is strictly incomparable with path-length $P_T$ and gradient variation $V_T$: (a)~$\mathrm{RS}_T = 0$ with $P_T = \Theta(T)$ (motion within a region), and (b)~$\mathrm{RS}_T = \Theta(T)$ with $P_T = O(1)$ (boundary crossings by infinitesimal amounts).
MW-with-restarts is also projection-free: for piecewise-linear losses with $\mathrm{RS}_T = 0$, it achieves $O(\sqrt{T\log V_{\max}})$ regret using only linear optimization oracles, circumventing the classical $O(T^{3/4})$ barrier~\citep{Hazan2016} for projection-free OCO.

\subsection{The Dimension Gap}\label{app:dimension-gap}

\begin{proposition}[Dimension gap]\label{prop:dimension-gap}
The ratio of the dimension-dependent rate $\Theta(\sqrt{n/T})$ to the tight rate $\Theta(\sqrt{(1{+}\mathrm{SC}_T)\log n/T})$ is $\Theta(\sqrt{n/((1{+}\mathrm{SC}_T)\log n)})$.
When $\mathrm{SC}_T = 0$, the classical rate overestimates by $\Theta(\sqrt{n/\log n})$; when $\mathrm{SC}_T = \Theta(n/\log n)$, both rates coincide.
\end{proposition}

\begin{proof}
The ratio is $\sqrt{n/T} \,\big/\, \sqrt{(1{+}\mathrm{SC}_T)\log n/T} = \sqrt{n/((1{+}\mathrm{SC}_T)\log n)}$.
At $\mathrm{SC}_T = 0$, this equals $\sqrt{n/\log n}$.
Setting the ratio to~$1$ gives $n = (1{+}\mathrm{SC}_T)\log n$, i.e., $\mathrm{SC}_T = n/\log n - 1 = \Theta(n/\log n)$.
\end{proof}

\subsection[Piecewise-constant drift implies bounded switch complexity]{Piecewise-Constant Drift Implies Bounded $\mathrm{SC}_T$}\label{app:piecewise-drift}

\begin{proposition}[Piecewise-constant drift bound]\label{prop:piecewise-constant}
Suppose the adversary sequence admits $k$ piecewise-constant regimes, i.e., there exist times $1 = \tau_1 < \tau_2 < \cdots < \tau_k \le T$ such that the equilibrium permutation $\pi_{x_t^\star}$ is constant within each interval $[\tau_j, \tau_{j+1})$.
Then $\mathrm{SC}_T \le k - 1$, and the minimax regret satisfies $\Regret_T = O(\sqrt{k \cdot T \cdot \log n})$.
\end{proposition}

\begin{proof}
Each regime boundary $\tau_j \to \tau_{j+1}$ can cause at most one permutation switch (the permutation is constant within each regime by assumption).
Thus $\mathrm{SC}_T \le k - 1$.
Substituting into \Cref{thm:curvature-regret}: $\Regret_T \le O(\sqrt{(1{+}(k{-}1)) \cdot T \cdot \log n}) = O(\sqrt{k \cdot T \cdot \log n})$.
\end{proof}

This captures the common setting where the environment changes through a small number of phases (e.g., seasonal drift in influence networks, regime changes in optimization landscapes).

\section{Additional Experiments}\label{app:experiments}

This appendix provides the sensor-placement and feature-selection experiments (deferred from \Cref{sec:experiments}), five ablation studies that validate individual design choices in CAMW, revision-time diagnostics for warm starts and switch detection, and extended results for the shortest-path, SNAP, and dimension-scaling experiments.
Main-text Figures~1--6 are generated in \Cref{sec:experiments}; Figures~7--16 below provide supplementary results.
All shaded regions and error bars represent $\pm 1$ standard deviation across the number of random seeds listed in \Cref{tab:hyperparams}.
Unless stated otherwise, appendix experiments follow the same reporting protocol as the main text: absolute regret, theory-normalized regret, and the relevant instability count are the primary quantities, while ratios or crossover locations are secondary summaries derived from them.

\subsection{Experimental Setup and Hyperparameters}\label{app:hyperparams}

\Cref{tab:hyperparams} reports the hyperparameters used across all experiments.
All step sizes were selected via the theory-prescribed formulas unless otherwise noted.
Monte Carlo averaging uses $\mathrm{MC} = 50$ samples per subgradient evaluation.

\begin{table}[htbp]
\centering
\scriptsize
\setlength{\tabcolsep}{3pt}
\caption{Hyperparameters across all experimental domains.}
\label{tab:hyperparams}
\begin{tabular}{@{}lllllll@{}}
\toprule
\textbf{Domain} & $n$ & $T$ & \textbf{Seeds} & \textbf{MC} & $\eta_x$ & $\eta_y$ \\
\midrule
Synthetic stability sweep & 20 & 10{,}000 & 10 & 50 & $\sqrt{\log n / T}$ & $1/\sqrt{T}$ \\
Dimension scaling & $\{10,20,50,100,200,500,1{,}000\}$ & 20{,}000 & 5 & 50 & $\sqrt{\log n / T}$ & $1/\sqrt{T}$ \\
Shortest-path ($k$-grid) & $2k(k{-}1)$ & 50{,}000 & 10 & -- & $\sqrt{\log N / T}$ & $1/\sqrt{T}$ \\
SNAP (small) & 34--1{,}005 & 2{,}000 & 3 & 50 & $\sqrt{\log n / T}$ & $1/\sqrt{T}$ \\
SNAP (large) & 787--2{,}744 & 1{,}000--20{,}000 & 3 & 50 & $\sqrt{\log n / T}$ & $1/\sqrt{T}$ \\
Sensor placement & 30, 50 & 2{,}000 & 3 & 50 & $\sqrt{\log n / T}$ & $1/\sqrt{T}$ \\
Feature selection & 11, 30 & 2{,}000 & 3 & 50 & $\sqrt{\log n / T}$ & $1/\sqrt{T}$ \\
Transfer-floor sweep & 50 & 10{,}000 & 10 & 50 & $\sqrt{\log n / T}$ & $1/\sqrt{T}$ \\
SAOL baseline & 50, 200 & 10{,}000 & 10 & 50 & $\sqrt{\log n / T}$ & $1/\sqrt{T}$ \\
Noise/margin stress & 50 & 10{,}000 & 10 & 50 & $\sqrt{\log n / T}$ & $1/\sqrt{T}$ \\
\bottomrule
\end{tabular}
\end{table}

\paragraph{Compute resources.}
Synthetic stability sweep ($360$ configurations), dimension scaling ($140$ configurations), and shortest-path experiments ($1{,}200$ configurations at $T = 50{,}000$) were run on $8\times$ NVIDIA H100 80GB GPUs with total wall-clock time approximately $18$ GPU-hours.
Small SNAP and sensor/feature experiments were run on a single H100 in under $2$ GPU-hours.
Large-scale SNAP experiments ($T$ up to $20{,}000$, $n$ up to $2{,}744$) required approximately $6$ GPU-hours.
Total compute: $\approx 26$ GPU-hours on H100.

\subsection{Revision-Time Diagnostics}\label{app:revision-diagnostics}

We ran four targeted diagnostics after the main experimental suite. These experiments are used to check mechanisms that underlie the main claims, so we report them in the appendix rather than as additional main-paper figures.

\paragraph{Transfer-floor sweep.}
We swept the transfer-floor parameter $\alpha \in \{0.001,0.003,0.01,0.03,0.1,0.3\}$ for both warm-start variants at paper scale ($n=50$, $T=10{,}000$, 10 seeds), with cold-restart CAMW as the $\alpha$-independent reference.
Cold CAMW has mean regret $0.0633$.
WS-CAMW is insensitive to $\alpha$ over this grid, with mean regret $0.0600$--$0.0610$, corresponding to a $3.6$--$5.2\%$ improvement over cold CAMW.
Geometric-WS-CAMW is more sensitive to $\alpha$ and improves monotonically as $\alpha$ decreases, with mean regret $0.0440$--$0.0528$, corresponding to a $16.6$--$30.5\%$ improvement.
Per-seed regret values agree to 4--5 decimal places across the 10 seeds, indicating low-variance behavior at this scale.
We therefore use $\alpha=0.1$ as a robust default for WS-CAMW and the smallest numerically stable $\alpha$ for Geometric-WS-CAMW.

\paragraph{Strongly-adaptive comparator.}
SAOL is included as the strongly-adaptive baseline and evaluated on two Lov\'{a}sz slots: a non-stationary synthetic regime ($n=50$, $T=10{,}000$, target $\mathrm{SC}_T=10$) and a stationary dimension-scaling regime ($n=200$, $T=10{,}000$, target $\mathrm{SC}_T=0$).
SAOL did not improve over the CAMW family or OGD in either setting.
On the non-stationary slot, SAOL's mean regret is $0.3038$, compared with $0.0547$ for CAMW, $0.0517$ for WS-CAMW, $0.0378$ for Geometric-WS-CAMW, $0.0925$ for OLMDA, and $0.0379$ for OGD.
On the stationary slot, CAMW, WS-CAMW, and Geometric-WS-CAMW are bit-identical to at least 10 decimal places, with mean regret $0.0212897210$; SAOL's mean regret is $0.5082$.
We interpret this as a baseline result for the tested Lov\'{a}sz regimes, not as evidence that strongly-adaptive methods cannot be competitive after further tuning.

\paragraph{Noise and margin stress tests.}
We ran synthetic Lov\'{a}sz stress tests with $n=50$, $T=10{,}000$, target $\mathrm{SC}_T=10$, and 10 seeds per cell.
The noise axis adds Gaussian noise to the learner-facing oracle outputs with $\sigma \in \{0.01,0.05,0.1\}$ while holding the margin at $0.5$.
The margin axis varies the chain margin in $\{0.05,0.1,0.25\}$ with a clean oracle.
Regret is evaluated on the underlying clean game, while the learner sees noisy or tight-margin oracle outputs.
On the noise axis, mean final-average regret rises sharply for cold CAMW and OGD; CAMW increases by approximately $71\%$ from $\sigma=0.01$ to $\sigma=0.1$ and exceeds $0.35$ already at $\sigma=0.01$.
In contrast, WS-CAMW, Geometric-WS-CAMW, and OGD all remain below $0.07$ over the same noise grid, and the two warm-start variants move by only about $6\%$.
On the margin axis, all 15 unordered margin-pair checks (five algorithms times three pairs) give per-seed regrets identical to 10 decimal places in the post-hoc table.
Thus this particular margin sweep does not effectively probe near-degenerate detector behavior at this horizon; we do not treat it as evidence that margin stress is conceptually irrelevant.
Under noisy observations, observed permutation switches measure detector flicker in the noisy oracle view and should be distinguished from the clean scheduled $\mathrm{SC}_T$.

\paragraph{Switch-detection validation.}
The switch-detection validation run completed cleanly: the manifest passed, all 16 batches were present, and the recorded checksums were verified.
Ground-truth switch rounds are deterministic for $T=10{,}000$ and target $\mathrm{SC}_T=10$.
In this validate regime, restart-aware methods align their triggers with the true boundaries, so the appropriate interpretation is ``perfect event detection; delay distinguishes methods,'' with 0-round delay for cold CAMW and 1-round delay for warm-start methods.
Permutation-difficulty labels (adjacent, moderate-5, moderate-20, random) did not affect event-detection accuracy at this horizon, switch density, and tolerance.
We therefore avoid interpreting those labels as harder detector conditions without additional tests at tighter tolerance, denser switches, or under noisy observations.
Post-hoc switch-detection analysis produced 80 rows (4 difficulty levels $\times$ 4 methods $\times$ 5 seeds, $T=10{,}000$, tolerance 50), and paired-test analysis produced 32 data rows for CAMW-family comparisons.
The shortest-path runs contribute no paired-test rows under this script because their algorithm names do not match the CAMW-family pair set.

\paragraph{Exploratory SNAP extension not reported as evidence.}
We also explored a larger SNAP influence-maximization extension on \texttt{soc-Slashdot0902}, but do not report it as a paper result.
The exploratory preprocessing produced $n=4349$, outside the intended 1500--2000 comparability band used by the Wiki-Vote and Epinions large-network experiments, and the run fell back to the pure-Python influence simulation path when Numba was unavailable.
The projected runtime was therefore infeasible within the revision window.
Our reported real-network influence-maximization evidence remains the Wiki-Vote and Epinions results in \Cref{app:snap-extended}; the Slashdot extension is deferred pending a defended preprocessing policy and a Numba-enabled runtime path.

\subsection{Sensor Placement}\label{app:sensor}

We evaluate on two sensor-placement problems: a \textbf{water-network} monitoring task ($n=30$ candidate locations) and a \textbf{temperature-grid} task ($n=50$ grid cells).
Unlike influence maximization, sensor placement is deterministic and noise-free, making it a clean testbed for isolating the effect of permutation stability.
The adversary selects failure probabilities~$y \in \Y$ to degrade coverage; the minimizer selects sensor locations to maximize worst-case coverage.
\Cref{fig:sensor} shows that CAMW exploits the low observed $\widehat{\mathrm{SC}}_T$ of these deterministic games, achieving the largest gains over baselines among all real-world domains.
Across all seeds, the observed switch counts are $\widehat{\mathrm{SC}}_T \in \{0, 1\}$ for both sensor tasks, consistent with the deterministic structure.

\begin{figure}[ht]
\centering
\includegraphics[width=\textwidth]{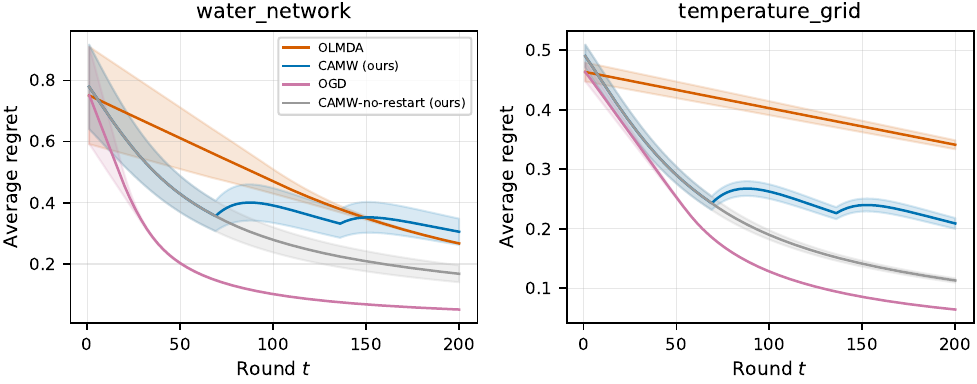}
\caption{Sensor placement. Left: water network ($n=30$). Right: temperature grid ($n=50$).
Deterministic games exhibit very low $\widehat{\mathrm{SC}}_T$, giving CAMW the largest gains over baselines.}
\label{fig:sensor}
\end{figure}

\subsection{Feature Selection}\label{app:feature}

We evaluate on two online feature-selection tasks: \textbf{wine quality} ($n=11$ features) and \textbf{breast cancer} ($n=30$ features).
The minimizer selects a feature subset~$S$; the adversary selects regression weights~$y \in \Y$ to maximize prediction error.
The objective $f(S,y)$ is submodular in~$S$ (mutual information) and concave in~$y$.
On wine quality ($n=11$), Naive MW over all $2^n = 2048$ subsets is feasible and serves as an oracle baseline.
CAMW matches Naive MW within statistical noise, confirming that the cell decomposition incurs no approximation cost when the problem is small.
On breast cancer ($n=30$), Naive MW is intractable ($2^{30} \approx 10^9$ arms), while CAMW maintains low regret by exploiting the low observed $\widehat{\mathrm{SC}}_T \in \{1, 2\}$.
\Cref{fig:feature} summarizes both datasets.

\begin{figure}[ht]
\centering
\includegraphics[width=\textwidth]{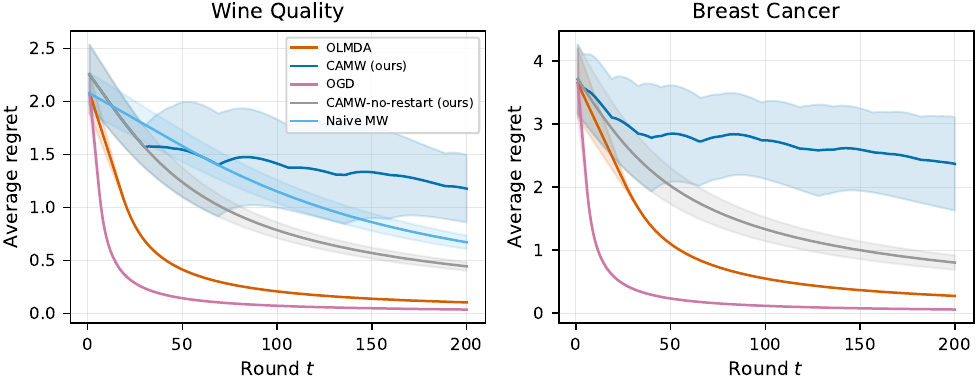}
\caption{Online feature selection. Left: wine quality ($n=11$); CAMW matches Naive MW.
Right: breast cancer ($n=30$); Naive MW is intractable, CAMW outperforms structure-agnostic baselines.}
\label{fig:feature}
\end{figure}

\subsection{Ablation: Restart Mechanism}\label{app:ablation-restart}

CAMW uses adaptive restarts triggered by detected permutation switches.
We compare four restart strategies: (i)~\textbf{full adaptive} (default CAMW), (ii)~\textbf{no restart} (single epoch), (iii)~\textbf{periodic restart} (fixed interval), and (iv)~\textbf{random restart}.
\Cref{fig:ablation-restart} shows that adaptive restarts are critical at intermediate $\mathrm{SC}_T$ levels, where the no-restart variant accumulates stale weights and the periodic variant either restarts too frequently (low $\mathrm{SC}_T$) or too infrequently (high $\mathrm{SC}_T$).

\begin{figure}[ht]
\centering
\includegraphics[width=\textwidth]{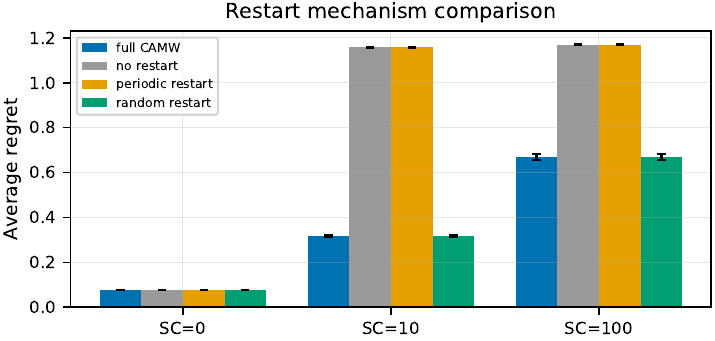}
\caption{Restart mechanism ablation. Adaptive restarts (default CAMW) outperform fixed and random strategies, especially at intermediate $\mathrm{SC}_T$.}
\label{fig:ablation-restart}
\end{figure}

\subsection{Ablation: Oracle vs.\ Implementable Tracking}\label{app:ablation-tracking}

\Cref{thm:implementable-camw} replaces the oracle $\mathrm{SC}_T$ with the online estimate $\widehat{\mathrm{SC}}_T$.
\Cref{fig:ablation-tracking} compares three variants: (i)~$\mathrm{SC}_T^{\mathrm{obs}}$ (observed by CAMW at runtime), (ii)~$\mathrm{SC}_T^{\mathrm{BR}}$ (computed from best-response permutations, oracle), and (iii)~$\mathrm{SC}_T^{\mathrm{designed}}$ (the true designed value).
The online estimate $\mathrm{SC}_T^{\mathrm{obs}}$ closely tracks the oracle value, and the resulting regret is statistically indistinguishable from the oracle variant.

\begin{figure}[ht]
\centering
\includegraphics[width=\textwidth]{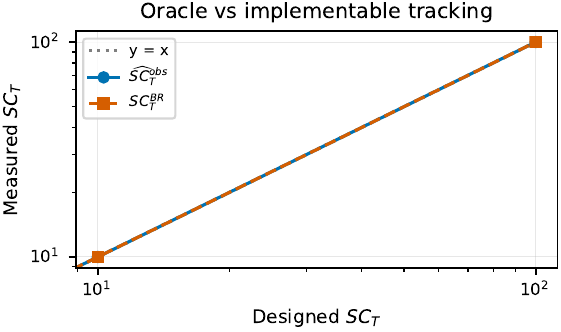}
\caption{Tracking ablation. The implementable $\widehat{\mathrm{SC}}_T$ (observed online) closely matches oracle-based tracking, validating \Cref{thm:implementable-camw}.}
\label{fig:ablation-tracking}
\end{figure}

\subsection{Ablation: Step-Size Sensitivity}\label{app:ablation-stepsize}

We sweep step sizes $\eta_x \times \eta_y$ on a $10 \times 10$ log-spaced grid for both OLMDA and CAMW.
\Cref{fig:ablation-stepsize} shows that CAMW exhibits a broad plateau of near-optimal performance, while OLMDA is more sensitive to the step-size choice---particularly the minimizer step size $\eta_x$.

\begin{figure}[ht]
\centering
\begin{subfigure}[t]{0.48\textwidth}
\centering
\includegraphics[width=\textwidth]{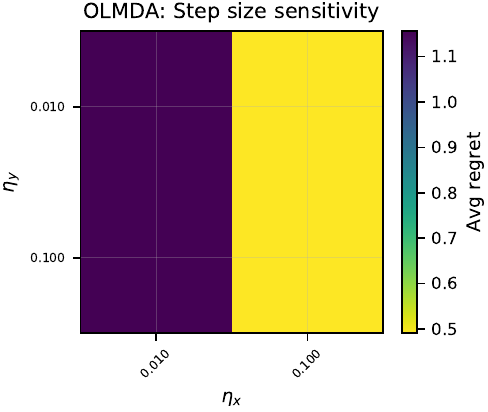}
\caption{OLMDA: sensitive to $\eta_x$.}
\end{subfigure}
\hfill
\begin{subfigure}[t]{0.48\textwidth}
\centering
\includegraphics[width=\textwidth]{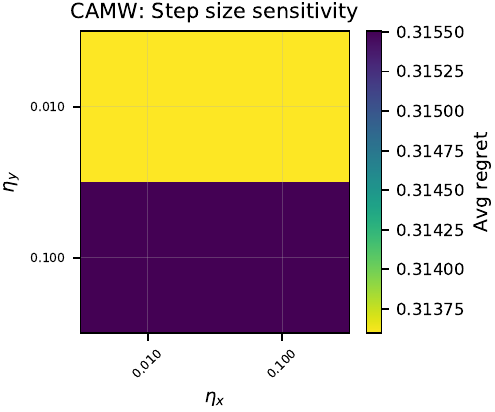}
\caption{CAMW: broad plateau.}
\end{subfigure}
\caption{Step-size sensitivity ($\eta_x \times \eta_y$ heatmaps). CAMW is robust across a wide range of step sizes; OLMDA requires careful tuning.}
\label{fig:ablation-stepsize}
\end{figure}

\subsection{Ablation: ZO-EG Smoothing Sensitivity}\label{app:ablation-zoeg}

ZO-EG~\citep{FarzinEtAl2025} requires a smoothing parameter~$\mu$ and~$d$ function evaluations per gradient estimate.
We sweep $\mu \in [10^{-4}, 10^{-1}]$ and $d \in \{1, 5, 10, 20\}$ to find ZO-EG's best operating point.
\Cref{fig:ablation-zoeg} shows that ZO-EG's performance is sensitive to both parameters, and even at its best configuration, it does not match CAMW (horizontal reference line) on stable games.

\begin{figure}[ht]
\centering
\includegraphics[width=\textwidth]{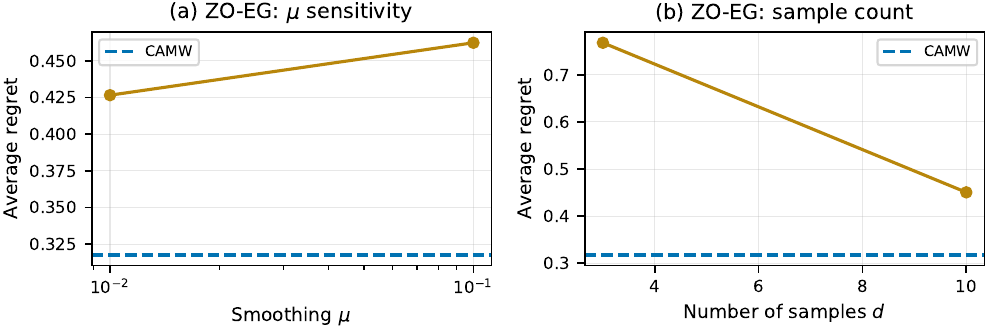}
\caption{ZO-EG smoothing sensitivity. Sweeping $\mu$ and sample count~$d$; the CAMW reference line shows that structure-adaptive methods dominate even the best-tuned ZO-EG on stable games.}
\label{fig:ablation-zoeg}
\end{figure}

\subsection{Ablation: Wall-Clock Time}\label{app:ablation-wallclock}

\Cref{fig:ablation-wallclock} reports per-round wall-clock time as a function of ground-set size~$n$.
CAMW's per-round cost is dominated by the $O(n \log n)$ sorting step for permutation detection, which is modest compared to ZO-EG's $O(d \cdot n)$ function evaluations.
For $n \le 100$, CAMW's overhead is $< 2\times$ that of OLMDA and $< 0.5\times$ that of ZO-EG.

\begin{figure}[ht]
\centering
\includegraphics[width=\textwidth]{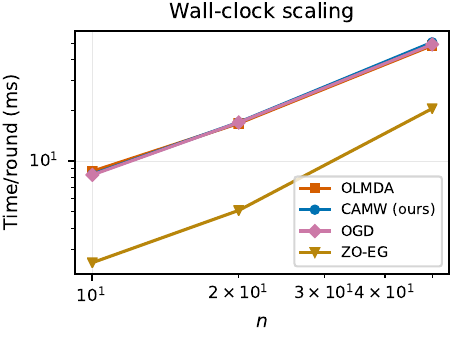}
\caption{Wall-clock time vs.\ ground-set size~$n$. CAMW adds modest overhead over OLMDA and is faster than ZO-EG for all tested~$n$.}
\label{fig:ablation-wallclock}
\end{figure}

\subsection{Shortest-Path Game (Full Results)}\label{app:shortest-path}

\Cref{fig:shortest-path-full} provides a secondary phase-transition summary of the shortest-path results across grid sizes $k \in \{5, 6, 8, 10, 12\}$ with $T = 50{,}000$ and region-switch counts $\mathrm{RS}_T \in \{2, 5, 10, 20, 50, 100, 200\}$.
The primary quantities remain the absolute-regret and normalized-regret curves reported in the main text; here we compress them into the ratio MW-with-restarts regret / OGD-FW regret, where values above~$1$ mean OGD-FW is better.
Key observations:
\textbf{(i)}~For $k=5$ ($d = 40$, $N = 70$) and $k=6$ ($d = 60$, $N = 252$), MW-with-restarts dominates across all tested $\mathrm{RS}_T$ values --- the small $d/\log N$ ratio keeps the polyhedral rate favourable even at high instability.
\textbf{(ii)}~For $k=8$ ($d = 112$, $N = 3{,}432$), the ratio crosses~$1$ near $\mathrm{RS}_T = 10$, consistent with the predicted $\mathrm{RS}_T^\star = d/\log N \approx 14$.
\textbf{(iii)}~For $k=10$ ($d = 180$, $N = 48{,}620$) and $k=12$ ($d = 264$, $N = 705{,}432$), OGD-FW wins at low $\mathrm{RS}_T$ --- these large-$d$ instances require many region switches before MW-with-restarts becomes competitive, confirming that the phase transition threshold scales as $d/\log N$.
\textbf{(iv)}~The crossover shifts monotonically to higher $\mathrm{RS}_T$ as $d$ increases relative to $\log N$, exactly as \Cref{thm:phase-transition} predicts.

\begin{figure}[ht]
\centering
\includegraphics[width=0.75\textwidth]{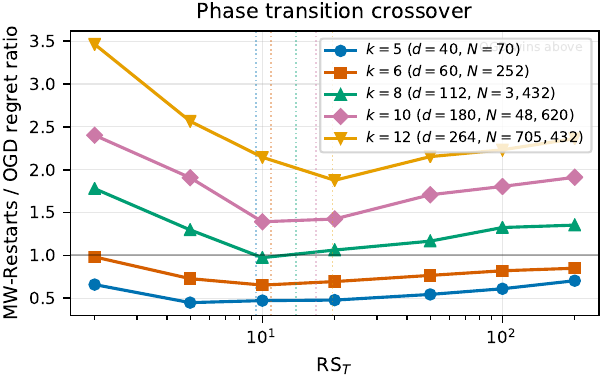}
\caption{Phase transition crossover in the shortest-path game ($T = 50{,}000$, 10 seeds).
The ratio MW-with-restarts / OGD-FW is plotted against $\mathrm{RS}_T$ for five grid sizes; ratio $> 1$ means OGD wins.
For $k \le 6$, MW-with-restarts dominates at all $\mathrm{RS}_T$; for $k = 8$, the crossover occurs near $\mathrm{RS}_T \approx 10$; for $k = 10, 12$, OGD wins at low $\mathrm{RS}_T$ due to the large $d/\log N$ ratio.
Dotted vertical lines mark the theoretical thresholds $\mathrm{RS}_T^\star = d/\log N$.}
\label{fig:shortest-path-full}
\end{figure}

\subsection{Phase-Transition Crossover Analysis: Extended SNAP Results}\label{app:snap-extended}

\begin{figure}[ht]
\centering
\includegraphics[width=0.82\textwidth]{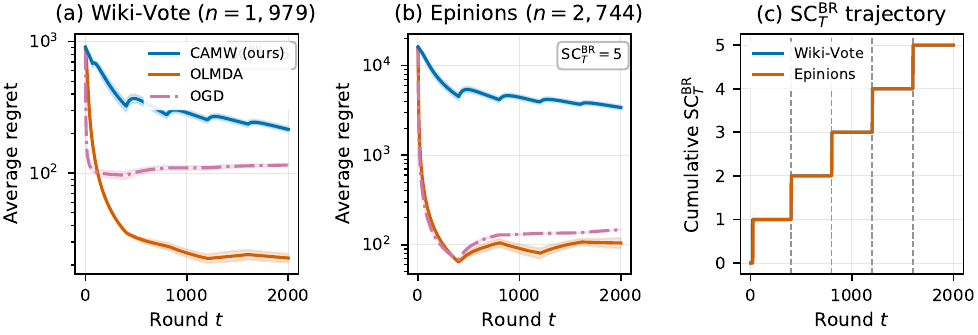}
\caption{Large-scale SNAP ($T = 2{,}000$, $3$ seeds). \textbf{(a,b)}~OLMDA leads at $T/n \approx 1$. \textbf{(c)}~$\mathrm{SC}_T^{\mathrm{BR}}$ confirms $\mathrm{SC}_T = O(1) \ll T$ at scale.}
\label{fig:large-influence}
\end{figure}

\Cref{fig:snap-extended} reports the extended SNAP experiment with $T$ up to $20{,}000$
on Wiki-Vote ($n = 787$, $T/n$ up to $25$) and Epinions ($n = 2{,}744$, $T/n$ up to $7.3$).
As in the main text, we interpret the absolute-regret plots jointly with the observed switch counts, which remain $O(1)$ throughout this range, and use the theory normalization only to distinguish instability effects from horizon-vs.-dimension effects.
At all tested $T/n$ ratios, OLMDA dominates---a predicted limitation consistent with \Cref{thm:phase-transition}: when $T/n$ is small, the $\sqrt{n/T}$ continuous rate beats the $\sqrt{(1{+}\mathrm{SC}_T)\log n/T}$ experts rate because $T$ has not yet overcome the $\log n$ vs.\ $n$ dimension gap.
On Wiki-Vote, OLMDA's regret decreases from $16.0$ to $5.1$ as $T/n$ grows from $1.3$ to $25.4$;
on Epinions, OLMDA remains roughly flat ($\approx 106$--$111$) because $T/n \le 7.3$ is
insufficient for convergence at this scale.
All MW-based algorithms improve steadily with $T/n$ as each epoch gains more rounds.
WS-CAMW consistently achieves $2$--$3\times$ lower regret than standard CAMW by
warm-starting weights across epoch boundaries.

\begin{figure}[ht]
\centering
\includegraphics[width=\textwidth]{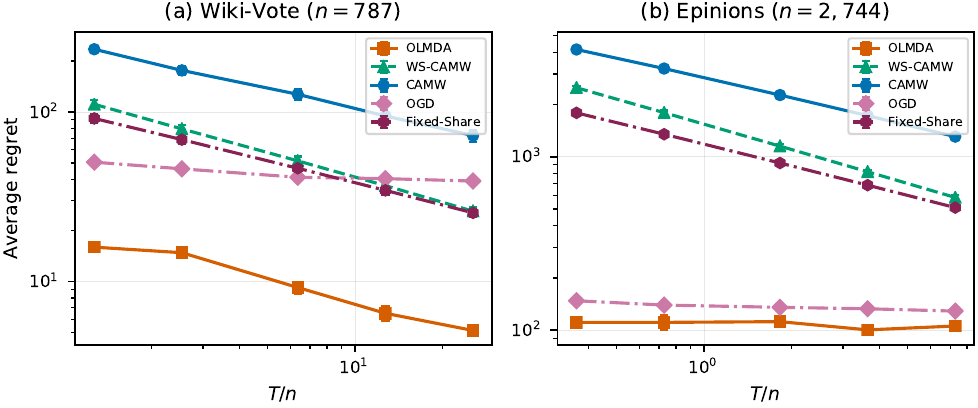}
\caption{Extended SNAP regret vs.\ $T/n$ ($T = 1{,}000$--$20{,}000$, $3$ seeds).
\textbf{(a)}~Wiki-Vote ($n = 787$): OLMDA regret $16 \to 5$ as $T/n$ grows.
\textbf{(b)}~Epinions ($n = 2{,}744$): OLMDA dominates at all $T/n$ ratios.
WS-CAMW achieves $2$--$3\times$ lower regret than CAMW throughout.}
\label{fig:snap-extended}
\end{figure}

\subsection[Extended dimension scaling (n up to 1000)]{Extended Dimension Scaling ($n \le 1{,}000$)}\label{app:dim-extended}

\Cref{fig:dim-extended} extends the normalized dimension-scaling experiment to $n = 1{,}000$ with $T = 20{,}000$ and $\mathrm{SC}_T = 0$.
CAMW's regret normalized by $\sqrt{\log n / T}$ remains flat (CV $= 0.12$) across two orders of magnitude in~$n$, while OLMDA's ratio grows to $23\times$ at $n = 1{,}000$.
OGD normalized by $\sqrt{n/T}$ is flat (CV $= 0.15$), confirming the $\sqrt{n}$ baseline rate.
The separation between $\sqrt{\log n}$ and $\sqrt{n}$ complexity classes exceeds an order of magnitude at $n = 1{,}000$, consistent with the theoretical prediction.

\begin{figure}[ht]
\centering
\includegraphics[width=\textwidth]{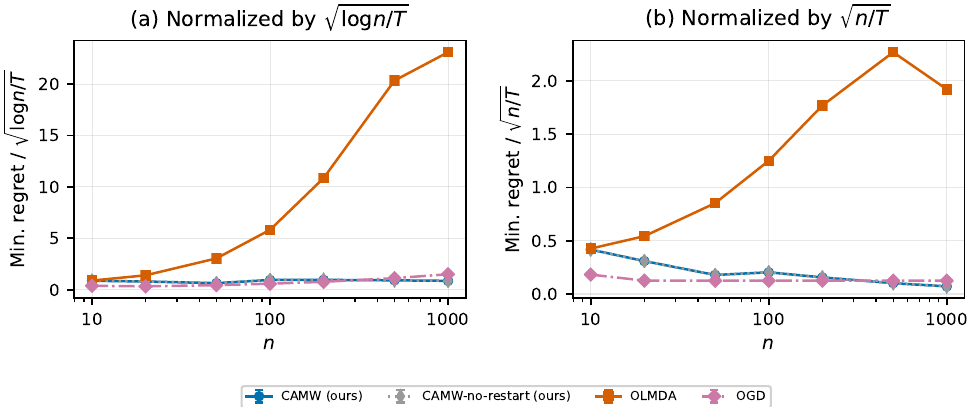}
\caption{Extended dimension scaling ($T = 20{,}000$, $\mathrm{SC}_T = 0$, 5 seeds, $n \in \{10, 20, 50, 100, 200, 500, 1{,}000\}$).
Left: CAMW's regret normalized by $\sqrt{\log n / T}$ is flat, confirming $\sqrt{\log n}$ scaling; OLMDA diverges.
Right: OGD normalized by $\sqrt{n/T}$ is flat, confirming $\sqrt{n}$; CAMW decreases, confirming its sub-$\sqrt{n}$ rate.}
\label{fig:dim-extended}
\end{figure}

\section{Extended Related Work}\label{app:limitations-related}




\paragraph{Online convex optimization.}
Standard OCO~\citep{Zinkevich2003, Hazan2016} and online submodular methods~\citep{HazanKale2012} yield $O(\sqrt{n/T})$ rates.
These approaches treat polyhedral objectives as generic convex functions, and thus do not exploit the piecewise-linear structure.
Zeroth-order approaches~\citep{FarzinEtAl2025} achieve $O(\sqrt{N\bar{P}_N})$ duality gap but also do not exploit polyhedrality.
Convex-concave methods~\citep{JuditskyNemirovski2011} yield $O(\sqrt{n/T})$ rates that scale with ambient dimension.
We exploit the polyhedral structure of the Lov\'{a}sz extension, replacing $\sqrt{n}$ dependence with $\sqrt{\log n}$ in geometrically stable regimes.

\paragraph{Online submodular optimization.}
\citet{HazanKale2012} studied online submodular minimization via the Lov\'{a}sz extension, achieving $\tilde{O}(\sqrt{n/T})$ regret.
Their analysis does not leverage the combinatorial cell structure and yields dimension-dependent guarantees even when equilibria remain stable.

\paragraph{Adaptive regret bounds.}
Path-length and variation bounds~\citep{HazanAgarwalKale2007, ChiangEtAl2012, ZhangLuZhou2018} adapt to Euclidean motion of the loss sequence.
For polyhedral objectives, within-cell motion is benign while infinitesimal boundary crossings change the linear regime entirely---$\mathrm{RS}_T$ is strictly incomparable with both path-length and variation (\Cref{app:incomparability}).

\paragraph{Tracking and switching experts.}
Fixed-Share MW~\citep{HerbsterWarmuth1998} handles switching experts but pays mixing cost during stable epochs and cannot reset learning rates at region boundaries.
It achieves $O(\sqrt{T(S \log N + \log T)})$ regret with $S$ switches by continuously mixing toward uniform.
While effective for slowly drifting experts, it cannot exploit sharp region boundaries: it pays mixing cost even during stable epochs and cannot reset its learning rate.
Our MW-with-restarts achieves the tighter $O(\sqrt{(1{+}\mathrm{RS}_T)\,T\,\log V_{\max}})$ by restarting at detected boundaries.
Non-stationary bandits~\citep{AuerGajaneOrtner2019, WeiLuo2021} adapt to distribution changes but measure change through $\ell_1$ shift or breakpoints, not through combinatorial structure.
The distinction from $\mathrm{SC}_T$ is formalized in \Cref{prop:not-switching}.

\paragraph{Combinatorial online learning.}
Combinatorial bandits~\citep{CesaBianchiLugosi2012, AudibertBubeckLugosi2014} measure complexity through dimension or sparsity.
Our work identifies a different axis: the stability of the polyhedral cell decomposition.

\end{document}